\documentclass{article} % For LaTeX2e
\usepackage{natbib}
\usepackage{macros}
\usepackage{authblk}
\usepackage{multirow}
\usepackage{times}
\usepackage{helvet}
\usepackage{courier}

\usepackage{macros}

\begin{document}

\title{A Study of Gradient Descent Schemes \\for General-Sum Stochastic Games}
\date{}

\author{H. L. Prasad\thanks{hlprasu@csa.iisc.ernet.in}}
\author{Shalabh Bhatnagar\thanks{shalabh@csa.iisc.ernet.in}}
\affil{\small Department of Computer Science and Automation, Indian Institute of Science, INDIA}
\date{}
\renewcommand\Authand{~ and }

\maketitle

\begin{abstract}
Zero-sum stochastic games are easy to solve as they can be cast as simple Markov decision processes. This is however not the case with general-sum stochastic games. A fairly general optimization problem formulation is available for general-sum stochastic games by \cite{FilarVrieze}. However, the optimization problem there has a non-linear objective and non-linear constraints with special structure. Since gradients of both the objective as well as constraints of this optimization problem are well defined, gradient based schemes seem to be a natural choice. We discuss a gradient scheme tuned for two-player stochastic games. We show in simulations that this scheme indeed converges to a Nash equilibrium, for a simple terrain exploration problem modelled as a general-sum stochastic game. However, it turns out that only global minima of the optimization problem correspond to Nash equilibria of the underlying general-sum stochastic game, while gradient schemes only guarantee convergence to local minima. We then provide important necessary conditions for gradient schemes to converge to Nash equilibria in general-sum stochastic games.
\end{abstract}


\section{Introduction}
Game theory is seen as a useful means to handle multi-agent scenarios. Since the seminal work of \cite{shapley}, stochastic games have been an important class of models for multi-agent systems. A comprehensive treatment of stochastic games under various payoff criteria is given by \cite{FilarVrieze}. Many interesting problems like fishery games, advertisement games, etc., can be modelled as stochastic games, see \cite{FilarVrieze}. One of the significant  results is that every stochastic game has a Nash equilibrium and it can be characterised in terms of global minima of a suitable mathematical programming problem. 

As an application of general-sum games to the multi-agent scenario, \cite{ssingh} observed that in a two-agent iterated general-sum game, Nash convergence is assured either in strategies or in the very least in average payoffs. Later by \cite{huwellman}, stochastic game theory was observed to be a better framework for multi-agent scenarios as it could be viewed as an extension of the well studied Markov decision theory (see \cite{bertsekas}). However, in the stochastic game setting, general-sum games are difficult to solve as, unlike zero-sum games, they cannot be cast in the framework similar to Markov decision processes. \cite{huwellman} proposed an interesting Q-learning algorithm, that is based on reinforcement learning (see \cite{berttsit}). However, their algorithm assures convergence only if the game has exactly one Nash equilibrium. %Additionally, some flaws in their proof of convergence were pointed out \cite{bowling} requiring some more assumptions which restricted the applicability of their algorithm further. 
An extension to the above algorithm called NashQ was proposed by \cite{huwellman2003} which showed improvement in performance. However, convergence in a scenario with multiple Nash equilibria was not addressed. Another noteworthy work is of \cite{littman} who proposed a friend-or-foe Q-learning (FFQ) algorithm as an improvement over the NashQ algorithm with assured convergence, though not necessarily to a Nash equilibrium. Moreover, the FFQ algorithm is applicable to a restricted class of games where either full co-operation between agents is ensured or the game is zero-sum. Algorithms for some specific cases of stochastic games such as Additive Reward and Additive Transition (AR-AT) games are discussed by \cite{FilarVrieze} as well as \cite{breton1986computation}. 

A new type of approach based on homotopy is proposed by \cite{herings2004stationary}. In this approach, a homotopic path between equilibrium points of $N$ independent MDPs and the $N$-player stochastic game in question, is traced numerically giving a Nash equilibrium point of the stochastic game of interest. Their result applies to both normal form as well as extensive form games. However, this approach has a complexity similar to that of typical gradient descent schemes discussed in this paper. For more recent developments in this direction, see the work by \cite{herings2006homotopy} and \cite{borkovsky2010user}.

A recent approach for the computation of Nash equilibria is given by \cite{akchurina} in which a reinforcement learning type scheme is proposed. Though their experiments do show convergence in a large group of randomly generated games, a formal proof of convergence has not been provided.

For general-sum stochastic games, \cite[Section 4.3]{breton1986computation} provides an interesting optimization problem with non-linear objectives and linear constraints whose global minima correspond to Nash equilibria of the underlying general-sum stochastic game. However, since the objective is not guaranteed to be convex, simple gradient descent techniques might not converge to a global minimum. \cite{mac2009solving} formulate intermediate optimization problems, called Multi-Objective Linear Programs (MOLPs), to compute Nash equilibria as well as Pareto optimal solutions. However, as mentioned in that paper, the complexity of their algorithm scales exponentially with the problem size. Thus, their algorithm is tractable only for small sized problems with a few tens of states.

Another non-linear optimization problem for computing Nash equilibria in general-sum stochastic games has been given by \cite{FilarVrieze}. We begin with this optimization problem by discussing it in Section~\ref{sect:opt}. Gradient based techniques are quite common for solving optimization problems. In the optimization problem, gradients of (both) the objective and all constraints w.r.t. the value vector $\v$ and strategy vector $\pi$, are well defined. A possible solution approach is to apply simple gradient based techniques to solve these optimization problems. We look at a possible gradient descent scheme in Section \ref{sect:grad-tech:scheme}. In the process of construction of this scheme, several initial hurdles for gradient descent schemes are discussed and addressed. We then consider an example problem of terrain exploration, modelled as a general-sum stochastic game, in Section \ref{sect:grad-tech:terrain}. In the same section, we show via simulation that the gradient descent scheme of Section \ref{sect:grad-tech:scheme} does indeed give a Nash equilibrium solution to the terrain exploration problem. In our case, the optimization problem at hand has only global minima correspond to Nash equilibria of the underlying general-sum stochastic game as discussed later in Section~\ref{subsect:theoretical}. It is well known that gradient descent schemes can only guarantee convergence to local minima. But, in the optimization problems that we consider, global minima are desired. So, a question remains: Are simple gradient descent schemes good enough to give Nash equilibria in the aforementioned optimization problems? In other words, are there only global minimum points in these optimization problems, so that simple gradient descent schemes can easily work? We address this issue in Section \ref{sect:grad-tech:fallacies}. Finally, in Section~\ref{sect:conclusion} we provide the concluding remarks.

\section{The Optimization Problem}\label{sect:opt}

The framework of a general-sum stochastic game is described in Section \ref{subsect:stochasticgame}. A basic idea of the optimization problem is given in Section \ref{sect:opt:basic-idea}. The full optimization problem is then formulated in Section \ref{subsect:opf} for the infinite horizon discounted reward setting. Some important results by \cite{FilarVrieze} that are applicable here are then described.

\subsection{Stochastic Games}\label{subsect:stochasticgame}

A two-agent scenario is considered in the following formulation. One can in general consider an $N$-agent scenario for $N \geq 2$. We assume $N = 2$ only for notational simplicity. We interchangeably use the terms `agent' and `player' to mean the same entity in the description below.
We assume that the stochastic game terminates in a finite but random time. Hence, a discounted value framework in dynamic programming has been chosen for the optimization problem.

A stochastic game is described via a tuple $< \S , \A, p, r >$.
The quantities in the tuple are explained through the description below.\\
\begin{inparaenum}[(i)]
\item $\S$ denotes the state space.\\
\item $\A^i(x)$ denotes the action space for the \ith{i} agent, $i=1,2$. $\A(x) = \mathop\times\limits_{i = 1}^{2} \A^i(x)$, the Cartesian product, is the aggregate action space consisting of all possible actions of both agents when the state of the game is $x \in \S$.\\
\item $p(y|x,a)$ denotes the probability of going from the state $x \in \S$ at the current instant to $y \in \S$ at the immediate next instant when the action $a \in \A(x)$ is chosen.\\
\item Finally, $r(x,a)$ denotes the vector of reward functions of both agents when the state is $x \in \S$ and the vector of actions $a \in \A(x)$ is chosen.\\
\end{inparaenum}

For an infinite horizon discounted reward setting, a discount factor $0 < \beta < 1$ is also included in the tuple describing the game. As is clear from this definition, a stochastic game can be viewed as an extension of the single-agent Markov decision process.

A \textit{strategy} $\pi^i \defn \left < \pi^i_1, \pi^i_2, \dots, \pi^i_t, \dots \right >$ of the \ith{i} player in a stochastic game prescribes the action to be performed in each state at each time instant $t$, by that player. We denote by $\pi^i_t(\cdot)$ the action prescribed for the \ith{i} agent by the strategy $\pi^i_t$ at time instant $t$. The quantity `$\cdot$' in $\pi_t^i(\cdot)$, in general, corresponds to the entire history of states and actions of all agents up to the \ith[st]{(t-1)} instant and the current system state at the \ith{t} instant. Let the set of all possible strategies for the \ith{i} player be denoted by $\mathcal{F}^i$. A strategy $\pi^i$ of player $i$, is said to be a Markov strategy if $\pi_t^i$ depends only on the current state $x_t \in \S$ at time $t$. %, and may in general vary with $t$ (i.e., could prescribe a different action, for the same state, at different times). 
Thus, for a Markov strategy $\pi^i$ of player $i$, $\pi^i_t(x) \in \A^i(x), \forall t \geq 0, x \in \S, i = 1, 2$. If the action chosen in any state for a Markov strategy $\pi^i$ is independent of the time instant $t$, viz., $\pi_t^i \equiv \bar\pi^i, \forall t \geq 0, i = 1, 2$, for some $\bar\pi$ such that $\bar\pi^i(x) \in \A^i(x), \forall x \in \S$, then the strategy is said be \textit{stationary}. Henceforth, we shall restrict our attention to stationary strategies only. By abuse of notation, we denote by $\pi$ itself the stationary strategy. Extending this to all players, we denote a Markov strategy-tuple by $\pi \defn \left < \pi^1, \pi^2 \right >$

Let $\Delta(\A(x))$ (resp. $\Delta(\A^i(x))$) denote the set of all probability measures on $\A(x)$ (resp. $\A^i(x)$). A randomized Markov strategy is specified via the sequence of maps $\phi_t^i: \S \rightarrow \Delta(\A^i(x))$, $x \in \S, t \geq 0, i = 1, 2$. Thus, $\phi^i_t(x)$ is a distribution on the set of actions $\A^i(x)$ and in general depends on time instant $t$. We say that $\phi_t^i$ is a \textit{stationary randomized strategy} or simply a \textit{randomized strategy} for player $i$ if $\phi_t^i \equiv \phi^i$. By an abuse of notation, we denote by $\pi = \left < \pi^1, \pi^2 \right >$, a stationary randomized strategy-tuple that we also (many times) call a strategy, since from now on, we shall only work with randomized strategies. We use $\pi^i(x, a)$ to denote the probability of picking action $a \in \A^i(x)$ in state $x \in \S$ by agent $i$. \cite[Theorem 3.8.1, pp. 130]{FilarVrieze} states that a Nash equilibrium in stationary randomized strategies exists for general-sum discounted stochastic games. We will refer to such stationary randomized strategies as {\em Nash strategies}. Similar to MDPs \citep{bertsekas}, one can define the value function as follows:
\begin{equation}
\label{eq:games:sg-value}
v^i_\pi(x_0) = E \left [\sum\limits_{t} \beta^t \sum_{a \in \A(x)} \left ( r^i(x_t, a) \prod_{i = 1}^2 \pi^i(x_t, a) \right ) \right ], \forall i = 1, 2.
\end{equation}
Let $\pi^{-i}$ represent the strategy of the agent other than the \ith{i} agent, that is, $\pi^{-1} = \pi^2$ and $\pi^{-2} = \pi^1$ respectively. Formally, we define Nash strategies and Nash equilibrium below.
\begin{definition}[Nash Equilibrium]
A stationary Markov strategy $\pi^* = \left < \pi^{1*}, \pi^{2*}\right >$ is said to be Nash if
\[v^i_{\pi^*}(x) \ge v^i_{\left <\pi^i, \pi^{-i*} \right>}(x), \forall \pi^i, i = 1, 2, \forall x \in \S.\] 
The corresponding equilibrium of the game is said to be a Nash equilibrium.
\end{definition}
Like in normal-form games \citep{nash1950equilibrium}, pure strategy Nash equilibria may not exist in the case of stochastic games. Using dynamic programming, the Nash equilibrium condition can be written as:
\begin{equation}
\label{eq:sg:dp}
v^i(x) = \max\limits_{\pi^i(x) \in \Delta(\A^i(x))} \left \{ E_{\pi(x)} \left [r^i (x, a) + \beta \sum \limits_{y \in U(x)} p(y|x, a) v^i(y) \right ] \right \}, \forall i = 1, 2.
\end{equation}
Unlike MDPs, \eqref{eq:sg:dp} involves two maximization equations. Note that the two equations are coupled because the reward of one agent is influenced by the strategy of the other agent and so is the state transition.

\subsection{The Basic Formulation}
\label{sect:opt:basic-idea}
%Let $\pi^{-i} = \left < \pi^1, \pi^2, \dots, \pi^{i - 1}, \pi^{i + 1}, \dots, \pi^N \right >$ represent strategies of all players except player $i$.
The dynamic programming equation \eqref{eq:sg:dp} for finding optimal values can now be revised to:
\begin{equation}
\label{eq:opt:basic-dp}
v^i(x) = \max\limits_{\pi^i(x) \in \Delta(\A^i(x))} \left \{ E_{\pi^i(x)} Q^i(x, a^i) \right \},\forall x \in \S, \forall i = 1, 2,
\end{equation}
where 
\[Q^i(x, a^i) = E_{\pi^{-i}(x)} \left [r^i (x, a) + \beta \sum \limits_{y \in U(x)} p(y|x, a) v^i(y) \right ],\]
represents the marginal value associated with picking action $a^i \in \A^i(x)$, in state $x \in \S$ for agent $i$. Also, $\Delta(\A^i(x))$ denotes the set of all possible probability distributions over $\A^i(x)$.
We derive a possible optimization problem from \eqref{eq:opt:basic-dp} in Section~\ref{sect:opt:basic-idea:objective} followed by a discussion of possible constraints on the feasible solutions in Section~\ref{sect:opt:basic-idea:constraints}.

\subsubsection{The objective}
\label{sect:opt:basic-idea:objective}
Equation \eqref{eq:opt:basic-dp} says that $v^i(x)$ represents the maximum value of $E_{\pi^i} Q^i(x, a^i)$ over all possible convex combinations of policy of agent $i$, $\pi^i \in \Delta(\A^i(x))$. However, neither the optimal value $v^i(x)$ nor the optimal policy $\pi^i$ are known {\em apriori}. So, a possible optimization objective would be 
\[f^i(\v^i, \pi^i) = \sum\limits_{x \in \S} \left ( v^i(x) - E_{\pi^i} Q^i(x, a^i) \right ),\]
which will have to be minimized over all possible policies $\pi^i \in \Delta(\A^i(x))$. But $Q^i(x, a^i)$, by definition, is dependent on strategies of all other agents. So, an isolated minimization of $f^i(\v^i, \pi^i)$ would really not make sense. Rather we need to consider the aggregate objective,
\[f(\v, \pi) = \sum\limits_{i = 1}^2 f^i(\v^i, \pi^i),\]
which is minimized over all possible policies $\pi^i \in \Delta(\A^i(x)), i = 1, 2.$ Thus, we have an optimization problem with objective as $f(\v, \pi)$ along with natural constraints ensuring that the policy vectors $\pi^i(x)$ remain as probabilities over all possible actions $\A^i(x)$ for all states $x \in \S$ for both agents. Formally, we write this optimization problem as below:
\begin{equation}
\label{eq:opt:basic-idea:first-cut-opt}
\left .
\begin{array}{l}
\min \limits_{\v,\pi} f(\v, \pi) = \sum \limits_{i = 1}^{2} \sum\limits_{x \in \S} \left (v^i(x) - E_{\pi^i} Q^i(x, a^i) \right ) \text{s.t.} \\
\subequationitem\label{subeq:opt:basic-idea:first-cut-opt:pi-ge-0} \pi^i(x, a^i) \ge 0, \forall a^i \in \A^i(x), x \in \S, i = 1, 2,\\
\subequationitem\label{subeq:opt:basic-idea:first-cut-opt:pi-sum-1} \sum \limits_{i = 1}^2 \pi^i(x, a^i) = 1, \forall x \in \S, i = 1, 2.
\end{array} \right \}
\end{equation}

Intuitively, all those $(\v, \pi)$ pairs which make $f(\v,\pi)$ as zero with $\pi$ satisfying \eqref{subeq:opt:basic-idea:first-cut-opt:pi-ge-0}-\eqref{subeq:opt:basic-idea:first-cut-opt:pi-sum-1}, should correspond to Nash equilibria of the corresponding general-sum discounted stochastic game. The question is: Is this true? We address this question in two parts. First, if $\pi^*$ represents a Nash strategy-tuple with $\v^*$ as the corresponding dynamic programming value obtained from \eqref{eq:sg:dp}, then we answer the question whether $f(\v^*,\pi^*)$ is zero? Second, if $(\v^*, \pi^*)$ is such that \eqref{subeq:opt:basic-idea:first-cut-opt:pi-ge-0}-\eqref{subeq:opt:basic-idea:first-cut-opt:pi-sum-1} are satisfied and $f(\v^*, \pi^*)$ is zero, then whether $\pi^*$ is a Nash strategy-tuple? We address these two questions in Lemmas \ref{lemma:opt:basic-idea:necessary} and \ref{lemma:opt:basic-idea:not-sufficient}.

\begin{lemma}
\label{lemma:opt:basic-idea:necessary}
Let $(\v^*, \pi^*)$ represent a possible solution for the dynamic programming equation \eqref{eq:sg:dp}. Then, $(\v^*, \pi^*)$ is a feasible solution of the optimization problem \eqref{eq:opt:basic-idea:first-cut-opt} and $f(\v^*, \pi^*) = 0$.
\end{lemma}
\begin{proof}
Proof follows simply from the construction of the optimization problem~\eqref{eq:opt:basic-idea:first-cut-opt}.
\end{proof}

\begin{lemma}
\label{lemma:opt:basic-idea:not-sufficient}
Let $(\v^*, \pi^*)$ be a feasible solution of the optimization problem~\eqref{eq:opt:basic-idea:first-cut-opt} such that $f(\v^*, \pi^*) = 0$. Then, $\pi^*$ need not be Nash strategy-tuple and $\v^*$ need not correspond to the dynamic programming value obtained from \eqref{eq:sg:dp}.
\end{lemma}
\begin{proof}
We provide a proof by example. Choose a $\pi^*$ such that it is not a Nash strategy-tuple. Then, to make $f(\v^*, \pi^*) = 0$, we need to compute a $\v^*$ such that
\begin{equation}
\label{eq:opt:basic-idea:first-cut:simple-dp-value}
v^{i*}(x) - E_{\pi^{i*}(x)} Q^i(x, a^i) = 0, \forall x \in \S, i = 1, 2.
\end{equation}
Let $R^i = \left < E_{\pi^*(x)} r^i(x, a): x \in \S \right >$ be a column vector over rewards to agent $i$ in various states of the underlying game. Also, let $P = \left [ E_{\pi^*(x)} p(y|x, a): x \in \S, y \in \S\right ]$ represent the state-transition matrix of the underlying Markov process. Then, \eqref{eq:opt:basic-idea:first-cut:simple-dp-value} can be written in vector form as 
\begin{equation}
\label{eq:opt:basic-idea:first-cut:simple-dp-value-vector}
\v^{i*} - \left ( R^i + \beta P \v^{i*} \right ) = 0, \forall i = 1, 2.
\end{equation}
Since $P$ is a stochastic matrix, all its eigen-values are less than or equal to one. Thus, the matrix $I - \beta P$ is invertible. So the system of equations \eqref{eq:opt:basic-idea:first-cut:simple-dp-value-vector} has a unique solution with
\[\v^{i*} = \left ( I - \beta P \right )^{-1} R^i, i = 1, 2.\]
Thus for any strategy-tuple $\pi^*$, which need not be Nash, there exists a corresponding $\v^*$ such that $f(\v^*, \pi^*) = 0$.
\end{proof}

\subsubsection{Constraints}
\label{sect:opt:basic-idea:constraints}
The basic optimization problem \eqref{eq:opt:basic-idea:first-cut-opt} has only a set of simple constraints ensuring that $\pi$ remains a valid strategy. As shown in lemma \ref{lemma:opt:basic-idea:not-sufficient}, this optimization problem is not sufficient to accurately represent Nash equilibria of the underlying general-sum discounted stochastic game. Here, we look at a possible set of additional constraints which might make the optimization problem more useful. Note that the term being maximized in equation \eqref{eq:opt:basic-dp}, i.e., $E_{\pi^i} Q^i(x, a^i)$, represents a convex combination of the values of $Q^i(x, a^i)$ over all possible actions $a^i \in \A^i(x)$ in a given state $x \in \S$ for a given agent $i$. Thus, it is implicitly implied that \[Q^i(x, a^i) \le v^i(x), \forall a^i \in \A^i(x), x \in \S, i = 1, 2, \dots, N.\]
So, we could consider a new optimization problem with these additional constraints. However, the previously posed question remains: Is this good enough to make $f(\v, \pi) = 0$, for a feasible $(\v, \pi)$ to correspond to a Nash equilibrium? We show that this is indeed true in the next section.

\subsection{Optimization Problem for two-player Stochastic Games}\label{subsect:opf}

An optimization problem on similar lines as in Section~\ref{sect:opt:basic-idea}, for a two-player general-sum discounted stochastic game has been given by \cite{FilarVrieze}. The optimization problem is as follows:
\begin{equation}
\label{eq:opt:2-player-opt}
\left .
\begin{array}{l}
\min \limits_{\v,\pi} f(\v, \pi) = \sum \limits_{i = 1}^{2} {\underline{1}}_{|\S|}^T \left [ \v^i - \r^i(\pi) - \beta P(\pi) \v^i \right ] \text{s.t.} \\
\subequationitem\label{subeq:opt:2-player:agent1} \pi^2(x)^T \left [ \r^1(x) + \beta \sum \limits_{y \in U(x)} P(y|x) v^1(y) \right ] \leq v^1(x) {\underline{1}}_{m^1(x)}^T \text{ } \forall x \in \S \\
\subequationitem\label{subeq:opt:2-player:agent2} \left [ \r^2(x) + \beta \sum \limits_{y \in U(x)} P(y|x) v^2(y) \right ] \pi^1(x)  \leq v^2(x) \underline{1}_{m^2(x)} \text{ } \forall x \in \S \\
\subequationitem\label{subeq:opt:2-player:agent1strategy} \pi^1(x)^T\underline{1}_{m^1(x)} = 1 \text{ } \forall x \in \S \\
\subequationitem\label{subeq:opt:2-player:agent2strategy} \pi^2(x)^T\underline{1}_{m^2(x)} = 1 \text{ }\forall x \in \S \\
\subequationitem\label{subeq:opt:2-player:agent1prob} \pi^1(x, a^1) \geq 0 \text{ } \forall a^1 \in \A^1(x)~\forall x \in \S \\
\subequationitem\label{subeq:opt:2-player:agent2prob} \pi^2(x, a^2) \geq 0 \text{ } \forall a^2 \in \A^2(x)~\forall x \in \S.
\end{array} \right \}
\end{equation}
where,
\begin{inparaenum}[\\\bfseries (i)]
\item $\v = \left < \v^i : i = 1, 2 \right >$ is the vector of value vectors of all agents with $\v^i = \left < v^i(y) : y \in \S \right >$ being the value vector for the \ith{i} agent (over all states). Here, $v^i(x)$ is the value of the state $x \in \S$ for the $i^{th}$ agent.
\item $\pi = \left < \pi^i : i = 1, 2 \right >$ and $\pi^i = \left < \pi^i(x) : x \in \S \right >$, where $\pi^i(x) = \left < \pi^i(x,a) : a \in \A^i(x) \right >$ is the randomized policy vector in state $x \in \S$ for the \ith{i} agent. Here $\pi^i(x, a)$ is the probability of picking action $a$ by the \ith{i} agent in state $x$.
\item $\r^i(x) = \left [ r^i(x,a^1,a^2): a^1 \in \A^1(x), a^2 \in \A^2(x) \right ]$ is the reward matrix for the \ith{i} agent when in state $x \in \S$ with rows corresponding to the actions of the second agent and columns corresponding to that of the first. Here, $r^i(x,a^1,a^2)$ is the reward obtained by the \ith{i} agent in state $x \in \S$ when the first agent has taken action $a^1 \in \A^1(x)$ and the second agent $a^2 \in \A^2(x)$.
\item $\r^i(x,a^1,\A^2(x))$ is the column in $r^i(x)$ corresponding to the action $a^1 \in \A^1(x)$ of the first player. Each entry in the column corresponds to one action of the second player which is why we use $\A^2(x)$ as an argument above. Likewise, we have $\r^i(x,\A^1(x),a^2)$ is the row in $\r^i(x)$ corresponding to the action $a^2 \in \A^2(x)$ of the second player.
\item $\r^i(\pi) = \r^i(\left < \pi^1, \pi^2 \right >) = \left < \pi^2(x)^T \r^i(x) \pi^1(x) : x \in \S \right >$, where $\pi^2(x)^T \r^i(x) \pi^1(x)$ represents the expected reward for the given state $x$ when actions are selected by both agents according to policies $\pi^1$ and $\pi^2$ respectively.
\item $P(y|x) = [ p(y|x,a): a = \left <a^1, a^2\right >, a^1 \in \A^1(x), a^2 \in \A^2(x) ]$ is a matrix representing the probabilities of transition from the current state $x \in \S$ to a possible next state $y \in \S$ at the next instant with rows representing the actions of the second player and columns representing those of the first player.
\item $P(y|x,a^1,\A^2(x))$ is the column in $P(y|x)$ corresponding to the case when  the first player picks action $a^1~\in~\A^1(x)$. As with $r^i(x,a^1,\A^2(x))$, each entry in the above column corresponds to an action of the second player. Similarly, $P(y|x,\A^1(x),a^2)$ is the row in $P(y|x)$ corresponding to the case when the second player picks an action $a^2~\in~\A^2(x)$.
\item $P(\pi) = P(\left < \pi^1, \pi^2 \right >)=\left [ \pi^2(x)^T P(y|x) \pi^1(x) : x \in \S, y \in \S \right ]$ is a matrix with columns representing the possible current states $x$ and rows representing the future possible states $y$. Here, $\pi^2(x)^T P(y|x) \pi^1(x)$ represents the transition probability from $x$ to $y$ under policy $\pi$.
\item $m^i(x) = |\A^i(x)|$, and (recall that)
\item $U(x) \subseteq \S$ represents the set of next states for a given state $x \in \S$.
\end{inparaenum}

The inequality constraints in the optimization problem are quadratic in $\v$ and $\pi$. The first set of inequality constraints (\ref{subeq:opt:2-player:agent1}) on the first agent are quadratic in $\v^1$ and $\pi^2$ and the second set (\ref{subeq:opt:2-player:agent2}) on the second agent are quadratic in $\v^2$ and $\pi^1$ respectively. However, in both cases, the quadratic terms are only cross products between the components of a value vector and a strategy vector.

The objective function is a non-negative cubic function of $\v$ and $\pi$. All the terms in the objective function consist of only cross terms. The cross terms between the value vector of an agent and the strategy vector of either agent are present in the term $P(\pi) \v^i$ of the objective function and those between strategy vectors of the two agents are in the term $\r^i(\pi)$.

We modify the constraints in the optimization problem (\ref{eq:opt:2-player-opt}) by eliminating all the equality constraints in it, as follows: One of the elements $\pi^i(x, a^i), a^i \in \A^i(x),$ in each of the equations ${\underline{1}}_{m^i(x)}^T\pi^i(x, a^i) = 1$ can be automatically set, thereby resulting in inequality constraints over the remaining components as below. Let $a^i(x), i=1,2$ denote the actions eliminated using the equality constraint. Then, the set of constraints can be re-written as in (\ref{eq:grad-tech:2-player-opt-inequality}).
\begin{equation}
\label{eq:grad-tech:2-player-opt-inequality}
\left .
\begin{array}{l}
\pi^2(x)^T \left [ r^1(x) + \beta \sum \limits_{y \in U(x)} P(y|x) v^1(y) \right ] \leq v^1(x) {\underline{1}}_{m^1(x)}^T \text{ }\forall x \in \S \\
\left [ r^2(x) + \beta \sum \limits_{y \in U(x)} P(y|x) v^2(y) \right ] \pi^1(x) \leq v^2(x) \underline{1}_{m^2(x)} \text{ } \forall x \in \S \\
\sum \limits_{\overline{a}^i \in \A^i(x)\backslash \{a^i(x)\}} \pi^i(x, \overline{a}^i) \leq 1 \qquad\forall x \in \S, i=1,2,\\
\pi^i(x, \overline{a}^i) \geq 0 \qquad\forall \overline{a}^i \in \A^i(x)\backslash \{a^i(x)\}~\forall x \in \S, i=1,2.
\end{array} \right \}
\end{equation}
The variables $\pi^i(x, a^i(x)) = 1 - \sum_{\overline{a}^i \in \mathcal{A}^i(x)\backslash \{a^i(x)\}} \pi^i(x, \overline{a}^i),\hspace{1ex}\forall x \in \mathcal{S}, i = 1,2$ are implicitly assigned in the above set of constraints. For the sake of simplicity, in the above, the equations related to the values $v^1$ and 
$v^2$ are written without performing elimination of the quantities $\pi^i(x, a^i)$. For further simplicity,
we represent all the inequality constraints in (\ref{eq:grad-tech:2-player-opt-inequality}) as $g_j(\v, \pi) \leq 0,~j = 1, 2,\dots n,$ where $n$ is the total number of constraints.

\subsection{Theoretical Results on the Optimization Problem}\label{subsect:theoretical}

The optimization problem described above is applicable for general-sum two-agent discounted stochastic games.  \cite[Theorems 3.8.1--3.8.3]{FilarVrieze} are given below as Theorems \ref{theorem:opt:fv1}-\ref{theorem:opt:fv3}. See \cite[pp.~130-132]{FilarVrieze} for a proof of these results.
\begin{theorem}
\label{theorem:opt:fv1}
In a general-sum, discounted stochastic game, there exists a Nash equilibrium in stationary strategies.
\end{theorem}

\begin{theorem}
\label{theorem:opt:fv2}
Consider a tuple $(\widehat{\v},\widehat{\pi})$. The strategy $\widehat{\pi}$ forms a Nash equilibrium for the general-sum discounted game if and only if $(\widehat{\v},\widehat{\pi})$ is the global minimum of the optimization problem with $f(\widehat{\v},\widehat{\pi}) = 0$.
\end{theorem}

Thus, the optimization problem defined in (\ref{eq:opt:2-player-opt}) has at least one global optimum having value zero which corresponds to the Nash equilibrium for the stochastic game.

\begin{theorem}
\label{theorem:opt:fv3}
Let $(\widehat{\v}, \widehat{\pi})$ be a feasible point for (\ref{eq:opt:2-player-opt}) with an objective function value $\gamma > 0$. Then $\widehat{\pi}$, forms an $\epsilon$-Nash equilibrium with $\epsilon \leq \dfrac{\gamma}{1 - \beta}$.
\end{theorem}

The above result in simple terms, says that, being in a small neighbourhood of a global optimal point of the optimization problem \eqref{eq:opt:2-player-opt} corresponds to being in a small neighbourhood of the corresponding Nash equilibrium. Thus, there is a correspondence between global optima and Nash equilibria. Thus, this is an important result from the point of view of numerical convergence behaviour.

\section{A Gradient Descent Scheme}
\label{sect:grad-tech:scheme}
The optimization problem \eqref{eq:opt:2-player-opt} for two-player general-sum stochastic games, has an interesting structure with only cross products between optimization variables appearing in both the objective function as well as the constraints. So, as the first naive way of handling this optimization problem, we see whether the same can be broken down into smaller problems via a {\em uni-variate} type scheme \cite[Section 5.4, pp. 350]{ssrao}. It is possible to see that the original problem can be split into two sets of linear optimization problems with \begin{inparaenum}[(i)] \item the first set having two optimization problems in $v^1$ and $v^2$ separately. Here, $\pi$ is held constant; and, \item the second having one in $\left <\pi^1(x), \pi^2(x)\right >$ for every possible state $x \in \S$. In each of these cases, $v$ is held constant. \end{inparaenum} Thus, with a uni-variate type of break down of the original problem, we get several smaller problems that can be easily solved. However, a major drawback of this approach is the inherent deficiency of the uni-variate methods which do not have guaranteed convergence in general. In fact, we observed in simulations and also through numerical calculations that this approach does indeed fail because of the above mentioned deficiency. Hence, we look at devising a non-linear programming approach. The algorithm to be discussed is mainly based on an interior-point search algorithm by \cite{twostagefeasible}.

We first discuss the difficulties posed by the optimization problem \eqref{eq:opt:2-player-opt}. We try to address these issues in the subsequent sections by presenting a suitable gradient-based algorithm. With a suitable initial feasible point, the iterative procedure of \cite{twostagefeasible} converges to a constrained local minimum of a given optimization problem. The unmodified algorithm of \cite{twostagefeasible} is presented in Section \ref{sect:grad-tech:herskovits}. Section \ref{sect:grad-tech:scheme:initial-point} discusses a scheme for finding an initial feasible point. Exploiting the knowledge about the functional forms of the objective and constraints, we present in Section \ref{sect:grad-tech:scheme:optimal-step} our modification in Herskovits algorithm to the procedure of selection of a suitable step length. And finally the modified algorithm in full, is provided in Section \ref{sect:grad-tech:scheme:algorithm}.

\subsection{Difficulties}\label{sect:grad-tech:scheme:difficulties}
We note that the optimization problem \eqref{eq:opt:2-player-opt} presents the following difficulties.
\begin{enumerate}
\item \textbf{Dimensionality -} The numbers of variables and constraints involved in the optimization problem are large. For the two agent scenario, the number of variables can be shown to be twice the sum of the cardinalities of the state and action spaces. For instance, in the terrain exploration problem discussed in Section \ref{sect:grad-tech:terrain}, for a simple $4 \times 4$ grid terrain with two agents and two objects, the number of variables is $(647 \times 2) + (4169 \times 2) = 9632$. The total number of inequality constraints for the same can also be computed to be $(4169 \times 2) + (4169 \times 2) = 16676$.
\item \textbf{Non-convexity -} The constraint region in the optimization problem is not necessarily convex. In fact, during simulations related to the terrain exploration problem (Section \ref{sect:grad-tech:terrain}), we observed that the condition does not hold for many constraints. So, in general, the optimization problem (\ref{eq:opt:2-player-opt}) has non-convex feasible region.
\item \textbf{Issue with steepest descent -} As explained in Section \ref{sect:opt:basic-idea}, the objective function in the optimization problem \eqref{eq:opt:2-player-opt} is obtained by averaging over strategies, the inequality constraint sets (\ref{subeq:opt:2-player:agent1}) and (\ref{subeq:opt:2-player:agent2}) respectively. This has an effect on the steepest descent gradient directions at the constraint boundaries. The steepest descent direction has been found to be always opposing the constraint boundaries. As a result, a gradient method with the steepest descent direction as its search direction will get stuck when it hits a constraint boundary.
\end{enumerate}

\subsection{The Herskovits Algorithm}
\label{sect:grad-tech:herskovits}
We observed that in the optimization problem (\ref{eq:opt:2-player-opt}), steepest descent directions most often oppose the active constraint boundaries. Hence a steepest descent direction cannot be used as it would get stuck at one such boundary point which may not be an optimal point. Herskovits method offers two features which address this issue: \begin{inparaenum}[(1)] \item The search direction selected at each iteration, while being a strictly descent direction, makes use of the knowledge of the gradients of constraints as well as the gradient of the objective; and \item the procedure is {\em strictly feasible}, i.e., at any iteration, the current best feasible point is not touching any constraint boundary. \end{inparaenum}

\paragraph*{Assumptions:}\ \\
The assumptions required for the two-stage method are as follows: 
\begin{enumerate}[(i)] 
\item The feasible region $\Omega$ has an interior $\Omega^o$ and is equal to the closure of $\Omega^o$, i.e., $\Omega = \bar \Omega^o$. 
\item Each $\left < \v, \pi \right > \in \Omega^o$, satisfies $g_i(\v, \pi) < 0, i = 1, 2, \dots, n$. 
\item There exists a real number $a$ such that the level set $\Omega_a = \{ \left < \v, \pi \right > \in \Omega | f(\v, \pi) \le a \}$ is compact and has an interior. 
\item The function $f$ is continuously differentiable and $g_j, j = 1,2,\dots,n$, is twice continuously differentiable in $\Omega_a$. 
\item \label{asm:grad-tech:active-set-gradients} At every $\left < \v, \pi \right > \in \Omega_a$, the gradients of active constraints form an independent set of vectors. 
\end{enumerate}

It can be seen that assumptions (i)-(iv) are easily verified considering the functional forms of the objective and constraints of the optimization problem \eqref{eq:opt:2-player-opt} and that state space, $\S$ and action space, $\A$, are assumed to be finite. Assumption (v) is carried over as it is. We present the algorithm of \cite{twostagefeasible} in two parts: First, we provide the two-stage feasible direction method in Algorithm \ref{algo:grad-tech:twostagefeasible} and then in Algorithm~\ref{algo:grad-tech:herskovits}, we present the full algorithm.

\begin{nonfloatalgorithm}{Two-stage Feasible Direction Method}
\label{algo:grad-tech:twostagefeasible}
\begin{algorithmic}
\PARAM $\alpha \in (0,1)$, $\rho_0 > 0$
\PARAM $w_j(v_0, \pi_0) > 0, j = 1,2,\dots,n$, continuous functions
\INPUT $\nabla f(v_0, \pi_0)$, $\nabla g_j(v_0,\pi_0), j = 1,2,\dots,n$
\OUTPUT $S$, a feasible direction
\vspace{0.5ex}
\hrule
\vspace{0.5ex}
\Step Set $\rho \leftarrow \rho_0$.	

\Step Compute $\gamma_0 \in \Re^{n}$, $S_0 \in \Re^{n}$ by solving the linear system
\begin{equation}
\label{eq:twostagefirst}
\left .\begin{array}{l}
S_0 = -\left[ \nabla f(v_0, \pi_0) + \sum\limits_{j = 1}^{n} \gamma_{0j} \nabla g_j(v_0, \pi_0)\right ]\\
S_0^T\nabla g_j(v_0, \pi_0) = - w_j(v_0,\pi_0)\gamma_{0j} g_j(v_0, \pi_0), j = 1,2,\dots,n 
\end{array}\right \}
\end{equation}
\Step Stop and output $S \leftarrow 0$ if $S_0 = 0$.

\Step Compute $\rho_1 = \dfrac{(1 - \alpha)}{\sum\limits_{i = 1}^{n} \gamma_{0i}}$, if $\sum\limits_{i = 1}^{n} \gamma_{0i} > 0$. Also, $\rho \leftarrow \frac{\rho_1}{2}$ if $\rho_1 < \rho$.

\Step Compute $\gamma \in \Re^n$ and $S \in \Re^n$ by solving the linear system
\begin{equation}
\label{eq:twostagesecond}
\left .\begin{array}{l}
S = -\left[ \nabla f(v_0, \pi_0) + \sum\limits_{j = 1}^{n} \gamma_{j} \nabla g_j(v_0, \pi_0)\right ]\\
S^T\nabla g_j(v_0, \pi_0) = - \left [ w_j(v_0,\pi_0)\gamma_{j} g_j(v_0, \pi_0) + \rho \|S_0\|^2 \right ], j = 1,2,\dots,n 
\end{array}\right \}
\end{equation}
where $\|S_0\|$ is the Euclidean norm of the direction vector $S_0$.
\Step Output $S$.
\end{algorithmic}
\end{nonfloatalgorithm}

This method computes a feasible direction in two stages. In the first stage (equation (\ref{eq:twostagefirst})), it computes a descent direction $S_0$. By using its squared norm as a factor, a feasible direction $S$ is computed in the second stage (equation (\ref{eq:twostagesecond})). Note that the second stage ensures that all the active constraints have $S^T\nabla g_j(v_0, \pi_0) = -\rho\|S_0\|^2$ where the right hand side is strictly negative. Thus, gradients of active constraints are maintained at obtuse angles with the direction $S$ and hence, the vector $S$ points away from the active constraint boundaries. Thus, feasibility of the direction $S$ gets ensured. Let $S = \left < S^1_v, S^2_v, S^1_\pi, S^2_\pi\right >$ be a descent search direction where $S_v^i$ (resp. $S_\pi^i$) is the search direction in $v^i$ (resp. $\pi^i$) for $i = 1,2$. Also, let $S_v = \left < S_v^1, S_v^2 \right >$ and $S_\pi = \left < S_\pi^1, S_\pi^2 \right >$. We now present the original Herskovits algorithm as Algorithm \ref{algo:grad-tech:herskovits}.

\begin{nonfloatalgorithm}{The Herskovits Algorithm}
\label{algo:grad-tech:herskovits}
\begin{algorithmic}
\PARAM $\nu > 1$, $\delta_0 \in (0, 1), \eta \in (0, 1)$
\INPUT $\left <v_0, \pi_0 \right >$: initial feasible point which is a strict interior point.
\OUTPUT $\left <v^*, \pi^*\right >$
\hrule\vspace{0.5ex}
\Step iteration $\leftarrow 1$.
\Step $\left <\widehat v, \widehat\pi \right > \leftarrow \left <v_0, \pi_0 \right >$
\LOOP
\Step Compute feasible direction $S$ using the two-stage feasible direction method (algorithm~\ref{algo:grad-tech:twostagefeasible}).
\Step Stop algorithm if $S = 0$. Set $\left <v^*, \pi^*\right > \leftarrow \left <\widehat{v}, \widehat{\pi}\right >$. Output $\left <v^*, \pi^*\right >$.
\Step\label{step:grad-tech:herskovits:delta} Let $\gamma = \left < \gamma_1, \gamma_2, \dots, \gamma_n \right >$, be the Lagrange multiplier vector computed in Algorithm~\ref{algo:grad-tech:twostagefeasible}. Define $\delta_j = \delta_0$ if $\gamma_j \ge 0$ and $\delta_j = 1$ if $\gamma_j < 0$.
\Step Find $t$, the first element in the sequence $\{1, 1/\nu, 1/\nu^2, \dots\}$ such that
\begin{equation}
\label{eq:grad-tech:herskovits:step-condition}
\begin{array}{l}
f(\v + t S_\v, \pi + t S_\pi) \le f(\v, \pi) + t \eta S^T \nabla f(\v, \pi),\text{ and }\\
g(\v + t S_\v, \pi + t S_\pi) \le \delta_i g(\v, \pi), \forall i = 1, 2, \dots, n.
\end{array}
\end{equation}
\Step Stop algorithm if $t = 0$. Set $\left <\v^*, \pi^*\right > \leftarrow \left <\widehat{\v}, \widehat{\pi}\right >$. Output $\left <\v^*, \pi^*\right >$.
\Step $\left <\widehat{\v}, \widehat{\pi}\right > \leftarrow \left <\widehat{\v}, \widehat{\pi}\right > + t S$
\Step iteration $\leftarrow$ iteration $+~1$.
\ENDLOOP
\end{algorithmic}
\end{nonfloatalgorithm}

This algorithm can be tuned by utilizing the knowledge about the structure of the optimization problem \eqref{eq:opt:2-player-opt}. We present the following modifications in this direction:
\begin{enumerate}
\item Computing the initial feasible point by a set of simple linear programs in Section \ref{sect:grad-tech:scheme:initial-point};
\item Exploiting the sparsity of the matrix involved in computing the two-stage feasible direction in Section \ref{sect:grad-tech:scheme:sparsity}; and
\item Knowing the cubic form of the objective and the quadratic form of constraints, and computing an optimal step-length in Section~\ref{sect:grad-tech:scheme:optimal-step}. However, we keep the condition \ref{eq:grad-tech:herskovits:step-condition} satisfied while selecting the step-length.
\end{enumerate}

\subsection{Initial Feasible Point}
\label{sect:grad-tech:scheme:initial-point}

The optimization problem given in (\ref{eq:opt:2-player-opt}) has a distinct separation between strategy probability terms and value vector terms. This can be exploited to find an initial feasible solution using the following procedure. First, a feasible strategy is selected, for instance, a uniform strategy, $\pi_0 = \left < \pi_0^i : i = 1,2 \right >$ with 
\begin{equation}
\label{eq:initial-pi}
\pi^i_0(x, a) = \dfrac{1}{m^i(x)} \quad \forall a \in \A^i(x), x \in \S, i = 1,2. % = \dfrac{1}{\left | \pi^i(x) \right |} = 
\end{equation}
If this strategy is held constant, then it is easy to see that the main optimization problem (\ref{eq:opt:2-player-opt}) breaks down into two linear programming problems in $v^1$ and $v^2$, respectively, as given in (\ref{eq:grad-tech:value-only}). For the Herskovits algorithm, a strict interior point is desired to start with. That is, the initial point for the algorithm needs to be strictly away from all constraint boundaries. So, we introduce a small positive parameter, $\alpha > 0$, in the left-hand side of the constraints given in \eqref{eq:grad-tech:value-only}.

\begin{equation}
\label{eq:grad-tech:value-only}
\left .\begin{array}{l}
\left .\begin{array}{c}
\min\limits_{v^1} \left \{ \underline{1}_{|\S|}^T\left (v^1 - r^1(\pi_0) - \beta P(\pi_0) v^1 \right ), \right \}\\
\text{s.t.}\qquad
\pi^2_0(x)^T \left [ r^1(x) + \beta \sum \limits_{y \in U(x)} P(y|x) v^1(y) \right ] + \alpha \leq v^1(x) {\underline{1}}_{m^1(x)}^T, \text{ }\forall x \in \S, \\
\end{array}\right ]\\
\vspace{1ex}\\
\left .\begin{array}{c}
\min\limits_{v^2} \left \{ \underline{1}_{|\S|}^T\left (v^2 - r^2(\pi_0) - \beta P(\pi_0) v^2 \right ) \right \},\\
\text{s.t.}\qquad
\left [ r^2(x) + \beta \sum \limits_{y \in U(x)} P(y|x) v^2(y) \right ] \pi^1_0(x)  + \alpha \leq v^2(x) \underline{1}_{m^2(x)} \text{ }\forall x \in \S.\\
\end{array}\right ]\\
\end{array}\right \}
\end{equation}

The two optimization problems in (\ref{eq:grad-tech:value-only}) can be solved readily with the popular method of revised simplex \cite[Chapter 7]{chvatal}. Since the main purpose here is to just get an initial feasible point, the first phase of the revised simplex method in which the auxiliary problem in an artificial variable is solved for, is in itself sufficient. The relevant details related to the first phase of the revised simplex method have been described in \cite[Chapter 7]{chvatal}. An initial feasible point $(v_0, \pi_0)$ can thus be obtained.

\subsection{Sparsity}\label{sect:grad-tech:scheme:sparsity}

The two-stage feasible direction method given in Algorithm~\ref{algo:grad-tech:twostagefeasible} requires inverting a matrix of dimension the same as the number of constraints.  Note that the number of constraints in the optimization problem is large (See Section~\ref{sect:grad-tech:scheme:difficulties}). Thus, in principle, our method would require a large memory and computational effort. However, we observe that the matrix to be inverted is sparse. Hence, we use efficient techniques for sparse matrices that result in a substantial reduction in the computational requirements. The matrix to be inverted is given by
\begin{equation}
\label{eq:sparse-matrix}
H = \left [ \begin{array}{cccc}
\nabla g_1^T \nabla g_1 - w_1 g_1 & \nabla g_1^T \nabla g_2 & \dots & \nabla g_1^T \nabla g_n \\
\nabla g_2^T \nabla g_1 & \nabla g_2^T \nabla g_2 - w_2 g_2 & \dots & \nabla g_2^T \nabla g_n \\
\vdots & \vdots & \ddots & \vdots \\
\nabla g_n^T \nabla g_1 & \nabla g_n^T \nabla g_2 & \dots & \nabla g_n^T \nabla g_n - w_n g_n \\
\end{array}
\right ].
\end{equation}

The elements of $H$ can be seen mainly to be dot products of constraint gradients. In a typical stochastic game, the number of states which are related by non-zero transition probability, is less compared to the total number of states. The same is applicable when we consider the action set. The action set available at each state is usually less overlapping with corresponding (action) sets of other states. In some cases, these sets may be completely disjoint as well. For the simple terrain exploration problem which shall be discussed in Section \ref{sect:grad-tech:terrain}, the above matrix for a $4\times4$ grid scenario with two objects and two agents, is of size $16676 \times 16676$, and is only about $4\%$ full.

We note that the two-stage feasible direction method does not really require an explicit inverse of the matrix. Rather, it requires the solution to the linear system of equations $H \gamma = b$, where $b$ is a vector of appropriate dimension. We target to use decomposition techniques for the purpose in which the matrix is decomposed as $H = L D L^T$ where $L$ is a lower triangular matrix and $D$ is a diagonal matrix. Since $H$ is also sparse, we do sparse $L D L^T$ decomposition of the matrix $H$ using techniques discussed by \cite{amd-soft,ldl}, using publicly available software on the internet. Upon decomposition of the matrix $H$, the solution to $\gamma$ can be easily computed.

\subsection{Computing the Optimal Step-length}
\label{sect:grad-tech:scheme:optimal-step}

The objective function has been shown previously to be cubic and constraints quadratic in the optimization variables. This structure can be exploited to find the optimal step length, $t^*$, in any chosen direction.

\subsubsection[Optimal Step Length]{Optimal Step Length, $t^*$}

Let $(v_0, \pi_0)$ be the current point and $(\v, \pi)$ be the next point obtained from the previous by moving one step along the descent direction. Thus, $v = v_0 + t S_v$ and $\pi = \pi_0 + t S_\pi$. Upon substitution into the objective function $f(\v, \pi)$, one obtains,
\begin{equation}
f(\v, \pi) = f(v_0 + t S_\v, \pi_0 + t S_\pi) = d_0 + d_1 t + d_2 t^2 + d_3 t^3
\end{equation}
where
{\small \begin{equation}
\label{eq:cubic-coeff}
\left .\begin{array}{l}
d_0 = \sum \limits_{i = 1}^{2} \underline{1}_{|\S|}^T \left \{ v^i_0 - r^i(\pi_0) - \beta P(\pi_0) v^i_0 \right \},\\ 
d_1 = \sum \limits_{i = 1}^{2} \underline{1}_{|\S|}^T \{ S^i_v - \left ( r^i(\left <\pi_0^1, S^2_\pi \right >) + r^i(\left < S^1_\pi, \pi_0^2\right >) \right ) \\\qquad\qquad\qquad\qquad\qquad\qquad - \beta \left ( P(\pi_0) S^i_\pi + P(\left <\pi^1_0, S^2_\pi\right >) v^i_0 + P(\left <S^1_\pi, \pi^2_0\right >) v^i_0 \right )\},\\
d_2 = \sum \limits_{i = 1}^{2} \underline{1}_{|\S|}^T \left \{ - r^i(S_\pi)- \beta \left ( P(\left <\pi^1_0, S^2_\pi\right >) S^i_v + P(\left <S^1_\pi, \pi^2_0\right >) S^i_v  + P(S_\pi) v^i_0\right ) \right \},\\
d_3 = \sum \limits_{i = 1}^{2} \underline{1}_{|\S|}^T \left \{ -\beta P(S_\pi) S^i_v \right \}.
\end{array}\right \}
\end{equation}}
In the above equations, the search direction terms $S_\pi^1$ and $S_\pi^2$ have been used in places where strategy terms are expected. Note that the search direction terms $S_\pi^1$ and $S_\pi^2$ do not form strategies. The usage here is purely in the functional sense.

Now, from $\dfrac{\partial f}{\partial t} = 0$, one obtains $d_1 + 2 d_2 t + 3 d_3 t^2 = 0$. Hence the extreme points of $f(\v, \pi)$ are given by $t^* = \dfrac{ - d_2 \pm \sqrt{d_2^2 - 3 d_1 d_3}}{3 d_3}$. \begin{inparaenum}[(i)] \item If $d_2^2 - 3 d_1 d_3 < 0$, then with increasing $t$, the function value shall decrease monotonically in the chosen direction $S$. So, any value $t \geq 0$ is fine without considering the constraints (see Figure \ref{fig:cubic}). \item If $d_2^2 - 3 d_1 d_3 \geq 0$, we get two extreme points in the chosen direction. If any of these points has a negative $t$ value, it is ignored. Since the direction is known to be descent, if one extreme point is negative, then so will be the other extreme point as well (see Figure \ref{fig:cubic}). \end{inparaenum} Till now only the objective function was considered. The approach to handle the constraints will be explained next.

\subsubsection[Constraints on step-length]{Constraints on step-length, $t$}\label{subsubsect:step-length-constriants}

The constraints in the optimization problem (\ref{eq:opt:2-player-opt}) impose limits on the possible values that $t$ can take. Consider the inequality constraint (\ref{subeq:opt:2-player:agent1}) for a particular $a^1 \in \A^1(x)$. Let $g_j(\cdot) \le 0$ represent one of these constraints and let $\delta_i$ represent the corresponding parameter computed in step \ref{step:grad-tech:herskovits:delta} of the Herskovits method (Algorithm~\ref{algo:grad-tech:twostagefeasible}). Apart from feasibility of step-size, we wish to ensure that the condition \eqref{eq:grad-tech:herskovits:step-condition}, i.e., $g_j(\v, \pi) \le \delta_j g_j(v_0, \pi_0),$ is ensured as well.  Now, substituting $v = v_0 + t S_v$ and $\pi = \pi_0 + t S_\pi$, and rearranging we get,
% \begin{equation}
% \left [ \pi^2_0(x) + t S_\pi^2(x) \right ]^T \left [ r^1(x, a^1) + \beta\sum \limits_{y \in U(x)} P(y|x, a^1) \left ( v^1_0(y) + t S^1_v(y) \right ) \right ]
% \leq \left ( v^1_0(x)  + t S^1_v(x) \right )
% \end{equation}
% On rearranging, we get, 
\begin{equation}
\label{eq:grad-tech:quadratic-constraint}
b^1(x, a^1) + c^1(x, a^1) t + d^1(x, a^1) t^2 \leq 0,
\end{equation}
where
{\small \begin{equation}
\label{eq:bcd-firstconstraint}
\left .\begin{array}{l}
b^1(x, a^1) = (1 - \delta_j) g_j(v_0, \pi_0),\\[1ex]
c^1(x, a^1) = \pi^2_0(x)^T \left [ \beta \hspace{-1em}\sum \limits_{y \in U(x)}\hspace{-1em} P(y|x, a^1) S_v^1(y) \right ] -S_v^1(x)\\\qquad\qquad\qquad\qquad\qquad  +S_\pi^2(x)^T \left [ r^1(x, a^1) + \beta \hspace{-1em}\sum \limits_{y \in U(x)}\hspace{-1em} P(y|x, a^1) v^1_0(y) \right ],\\[1ex]
d^1(x, a^1) = S_\pi^2(x)^T \left [ \beta\hspace{-1em} \sum \limits_{y \in U(x)}\hspace{-1em} P(y|x, a^1) S_v^1(y) \right ],\\
\end{array}\right \}
\end{equation}}
respectively. Let $\D(x, a^1) = c^1(x, a^1) c^1(x, a^1) - 4 b^1(x, a^1) d^1(x, a^1)$. Consider the case where $\D < 0$. This implies that the quadratic does not intersect the $t$-axis at any point i.e., it lies fully above the $t$-axis or fully below it. Clearly, $d^1(x, a^1) < 0$, implies that the quadratic lies fully below the $t$-axis and vice versa for $d^1(x, a^1) > 0$. 

\begin{proposition}
If $\D(x, a^1) < 0$, then $d^1(x, a^1) < 0$.
\end{proposition}
\begin{proof}
Suppose this is not true. Then $d^1(x, a^1) \geq 0$. Hence, $b^1(x, a^1) d^1(x, a^1) \leq 0$ because by definition of $\delta_j$ and $g_j(v_0, \pi_0)$, we have that $b^1 (x, a^1) \leq 0$. Thus, we obtain $c^1(x, a^1) c^1(x, a^1) - 4 b^1(x, a^1) d^1(x, a^1) \geq 0$ which is a contradiction. Hence, $d^1(x, a^1) < 0$. 
\end{proof}

Thus, for the case when $\D(x, a^1) < 0$, any value of $t \geq 0$ is fine as the quadratic is fully below the $t$-axis. Now consider the case where $\D(x, a^1) \ge 0$. On solving the quadratic function, we get its two roots as
\begin{equation}
t_1^1(x, a^1) = \dfrac{-c^1(x, a^1) + \sqrt{\D(x, a^1)}}{2 d^1(x, a^1)}\text{ and } t_2^1(x, a^1) = \dfrac{-c^1(x, a^1) - \sqrt{\D(x, a^1)}}{2 d^1(x, a^1)},
\end{equation}
$\forall a^1 \in \A^1(x)~\forall x \in \S$. If $t_1^1(x, a^1) \geq t_2^1(x, a^1)$, it can be shown that the region allowed by this constraint is given by the interval $[ t_2^1(x, a^1), t_1^1(x, a^1) ]$ in the given direction $S$. Otherwise for $t_1^1(x, a^1) < t_2^1(x, a^1)$, the region allowed by this constraint on the real line is given by the interval $\left (-\infty ,t_1^1(x, a^1) \right ]~\bigcup~\left [t_2^1(x, a^1), \infty \right )$. Note that this implies that the constraint is not convex. The above explanation can be easily adapted for constraints on the second agent as well. Thus, feasible value ranges for $t$ imposed by the constraints \eqref{subeq:opt:2-player:agent1} and \eqref{subeq:opt:2-player:agent2} can be obtained. We formalize in Algorithm~\ref{algo:grad-tech:quadratic-constraint}, this process of computing feasible value ranges for $t$ imposed by quadratic constraints \eqref{eq:grad-tech:quadratic-constraint}.

\begin{nonfloatalgorithm}{Feasible $x$ from a Quadratic Constraint} 
\label{algo:grad-tech:quadratic-constraint}
\begin{algorithmic}
\INPUT $b$, $c$, $d$ - Coefficients of quadratic constraint $b + c x + d x^2 \leq 0$.
\OUTPUT $L$ - Feasible $x$ set
\vspace{1ex}
\hrule
\vspace{1ex}
\IF{$d = 0$}
	\IF{$c = 0$}
		\IF{$b > 0$}
			\STATE $L = \phi$
		\ELSE
			\STATE $L = \Re$
		\ENDIF
	\ELSIF{$c > 0$}
		\STATE $L = [-b/c, \infty)$
	\ELSE
		\STATE $L = (-\infty, -b/c]$
	\ENDIF
\ELSE
	\STATE $\D = c^2 - 4bd$
	\IF{$\D < 0$}
		\IF{$d \geq 0$}
			\STATE $L = \phi$
		\ELSE
			\STATE $L = \Re$
		\ENDIF
	\ELSE
		\STATE $x_1 = \dfrac{-c + \sqrt{\D}}{2d}$ \COMMENT{Upper limit, $x \leq x_1$} 
		\STATE $x_2 = \dfrac{-c - \sqrt{\D}}{2d}$ \COMMENT{Lower limit, $x \geq x_2$}
		\IF{$x_2 \leq x_1$}
			\STATE $L = [x_2, x_1]$
		\ELSE
			\STATE $L = (-\infty, x_1]\bigcup[x_2, \infty)$
		\ENDIF
	\ENDIF
\ENDIF
\end{algorithmic}
\end{nonfloatalgorithm}
% 
% {\small\begin{equation}
% \label{eq:bcd-secondconstraint}
% \left .\begin{array}{l}
% b^2(x, a^2) = \left [ r^2(x, a^2) + \beta \hspace{-1em}\sum \limits_{y \in U(x)}\hspace{-1em} P(y|x, a^2) v^2_0(y) \right ] \pi^1_0(x) - v^2_0(x), \vspace{1ex}\\
% c^2(x, a^2) = \left [ \beta \hspace{-1em}\sum \limits_{y \in U(x)}\hspace{-1em} P(y|x, a^2) S_v^2(y) \right ] \pi^1_0(x) - S_v^2(x) + \left [ r^2(x, a^2) + \beta\hspace{-1em} \sum \limits_{y \in U(x)}\hspace{-1em} P(y|x, a^2) v^2_0(y) \right ] S_\pi^1(x), \vspace{1ex}\\
% d^2(x, a^2) = \left [ \beta\hspace{-1em} \sum \limits_{y \in U(x)}\hspace{-1em} P(y|x, a^2) S_v^2(y) \right ] S_\pi^1(x),\\
% \end{array}\right \}
% \end{equation}}
% respectively.

Equality constraints \eqref{subeq:opt:2-player:agent1strategy} and \eqref{subeq:opt:2-player:agent2strategy}, on the values of $\pi$ are $\underline{1}_{m^i(x)}^T\pi^i(x) = 1~\forall x\in\S, i = 1,2$. On using $\pi = \pi_0 + t S_\pi$ and $\underline{1}_{m^i(x)}^T\pi^i_0(x) = 1, \forall x\in\S,$ we get, \[t \times \left [ \underline{1}_{m^i(x)}^TS_\pi^i(x) \right ] = 0, \forall x\in\S,\] which do not impose any condition on the value of $t$. However, the set of inequality constraints \eqref{subeq:opt:2-player:agent1prob} and \eqref{subeq:opt:2-player:agent2prob}, for non-negativity of strategy vectors, i.e., $\pi^i(x,a^j) \geq 0, \forall a^j \in \A^i(x), x \in \S, i = 1,2$ may impose an upper limit on the value of $t$. For example, consider $\pi^i(x,a^j) \geq 0$. Let $\pi_0^i(x, a^j)$ represent the current best value of the probability of picking action $a^j$ in state $x$ by agent $i$, and let $\delta_j$ represent the corresponding parameter computed in step \ref{step:grad-tech:herskovits:delta} of the Herskovits algorithm, (see Algorithm~\ref{algo:grad-tech:herskovits}). We wish to satisfy the condition \eqref{eq:grad-tech:herskovits:step-condition}, i.e., $\pi^i(x, a^j) \ge \delta_j \pi_0^i(x, a^j)$. Upon substituting $\pi^i(x, a^j) = \pi_0^i(x, a^j) + t S_\pi^i(x, a^j)$, we get, $\pi_0^i(x, a^j) + t S_\pi^i(x, a^j) \geq \delta_j \pi_0^i(x, a^j)$. If $S_\pi^i(x, a^j) < 0$, then, \[t \leq \dfrac{\pi_0^i(x, a^j) (1 - \delta_j)}{-S_\pi^i(x, a^j)}\text{ or } t \in \left (-\infty, \dfrac{\pi_0^i(x, a^j) (1 - \delta_j)}{-S_\pi^i(x, a^j)}\right ].\] 
Note that $t \geq 0$ is implicitly assumed, else while $S$ is a descent direction,
$t S$ would not be one. Thus, if $S_\pi^i(x, a^j) \geq 0$, we get $t \geq \dfrac{\pi_0^i(x, a^j) (1 - \delta_j)}{-S_\pi^i(x, a^j)}$ which does not impose any additional constraint on $t$ and hence can be ignored. Intersection of feasible regions given by all constraints in \eqref{subeq:opt:2-player:agent1}, \eqref{subeq:opt:2-player:agent2}, \eqref{subeq:opt:2-player:agent1prob} and \eqref{subeq:opt:2-player:agent2prob}, gives the feasible set of values for $t$, from which a suitable step length $t^* \geq 0$ is selected. The procedure for doing so is explained next.

\subsubsection[Selection of the optimal step length]{Selection of the optimal step length, $t^*$}

Along the chosen descent direction $S$, the objective function $f(\v, \pi)$ is a cubic function in the step length $t$. If the extreme points are real and both positive, then the first sub-point in the descent direction will be a minimum point and the next a maximum point as shown in Figure \ref{fig:cubic}. So, under this condition, the best point is obtained by finding the best among the minimum point (or two feasible points near the minimum point) and the maximum step length point which is decided by the constraints. Otherwise, the cubic curve would be like the dashed curve in Figure \ref{fig:cubic}. In such a case, the optimal step length is simply the maximum feasible step length.

%In Figure \ref{fig:cubic}\subref{subfig:cubicwithpoints}, the feasible intervals of constraints on the value $t$ have been marked using dash-dotted lines. The point on the curve corresponding to $t = 0$ represents the current point. As can be seen from the figure, the possible values for $t$ in the positive direction are not limited without constraints. The intersection of the intervals defined by the constraints lead to a step length of $0.1312902$. Similarly, in Figure \ref{fig:cubic}\subref{subfig:cubicwithoutpoints} where the cubic curve does not have extreme points, the optimal value of $t$ obtained by taking the maximum possible step length in the intersection of the intervals is $0.129122$.

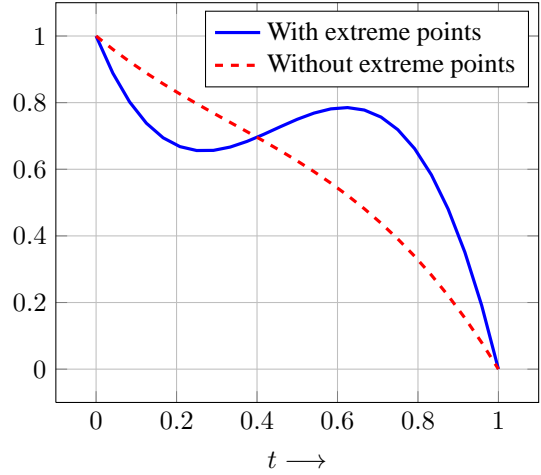
\begin{figure}
\centering
% \begin{tabular}{l}
%  \subfigure[With the two extreme points]{\label{subfig:cubicwithpoints}\includegraphics[scale=0.35]{../cubicwithpoints.pdf}}\\
%  \subfigure[Without any extreme points]{\label{subfig:cubicwithoutpoints}\includegraphics[scale=0.35]{../cubicwithoutpoints.pdf}}
% \end{tabular}
\begin{tikzpicture}
\begin{axis}[xlabel={$t\longrightarrow$},width=8cm,grid,no markers,tick scale binop={\times},legend style={ cells={anchor=west} }]
\addplot[color=blue,very thick,domain=0:1] {1 - 3*x + 8*x^2 - 6*x^3};
\addlegendentry{With extreme points}
\addplot[color=red,dashed,very thick,domain=0:1] {1 - x + x^2 - x^3};
\addlegendentry{Without extreme points}
\end{axis}
\end{tikzpicture}
 \caption{Cubic Curves (See Section \ref{subsubsect:step-length-constriants} for details)}
\label{fig:cubic}
\end{figure}

% %\begin{multicols}{1}
% \begin{figure}
%  %\centering
% \begin{tabular}{l}
%  \subfigure[With the two extreme points]{\label{subfig:cubicwithpoints}\includegraphics[scale=0.35]{cubicwithpoints.pdf}}\\
%  \subfigure[Without any extreme points]{\label{subfig:cubicwithoutpoints}\includegraphics[scale=0.35]{cubicwithoutpoints.pdf}}
% \end{tabular}
%  \caption{Cubic Curves (See Section \ref{sect:grad-tech:scheme:selection-optimal-step} for details)}
% \label{fig:cubic}
% \end{figure}
%\end{multicols}
%\FloatBarrier

\subsubsection{Optimal Step Length Algorithm}
\begin{nonfloatalgorithm}{Step Length Calculation}
\label{algo:grad-tech:steplength}
\begin{algorithmic}
\PARAM $\beta$: discount factor
\INPUT $(v_0, \pi_0)$: current value strategy pair
\INPUT $S$: selected descent direction 
\OUTPUT $t$: The best step length
\hrule\vspace{1ex}
\Step Calculate $d_1$, $d_2$ and $d_3$ using (\ref{eq:cubic-coeff}).
\Step $(t_1, t_2) \leftarrow roots(d_1, 2d_2, 3d_3)$

\Step $F \leftarrow \Re^+$, the set of all non-negative real numbers
\FOR{$x \in \S, a^1 \in \A^1(x), a^2 \in \A^2(x), i = 1, 2$}
\Step $F \leftarrow F \cap quadraticfeasible(b^i(x, a^i), c^i(x, a^i), d^i(x, a^i))$ where $b^i(x, a^i)$, $c^i(x, a^i)$, $d^i(x, a^i)$ are from (\ref{eq:bcd-firstconstraint}).
\Step $F \leftarrow F \cap \left [0, \dfrac{\pi_0^i(x, a^i)}{-S_\pi^i(x, a^i)}\right ]$ if $S_\pi^i(x, a^i) < 0$.
\ENDFOR
\Step If the extreme points $t_1$ and $t_2$ are real and both positive, the best step $t$ is obtained by finding the best amongst the minimum point $t_1$ (or two feasible points in $F$ near the minimum point) and the maximum step in $F$. Otherwise, the best step $t$ is the maximum step in $F$.
\end{algorithmic}
\end{nonfloatalgorithm}

In the above, $roots(a,b,c)$ gives the roots of $a + bx + cx^2 = 0$. Algorithm~\ref{algo:grad-tech:quadratic-constraint} is being referred to as $quadraticfeasible(a,b,c)$.

\subsection{The Complete Algorithm}\label{sect:grad-tech:scheme:algorithm}
With the schemes discussed in Sections \ref{sect:grad-tech:scheme:initial-point}, \ref{sect:grad-tech:scheme:sparsity} and \ref{sect:grad-tech:scheme:optimal-step}, we present the modified Herskovits algorithm.
\begin{nonfloatalgorithm}{The Complete Algorithm}
\label{algo:grad-tech:scheme}
\begin{algorithmic}
\PARAM $\beta$: discount factor
\INPUT $\pi_0$: initial strategy (from (\ref{eq:initial-pi}))
\OUTPUT $\left <v^*, \pi^*\right >$: An $\epsilon$-Nash equilibrium with $\epsilon = \dfrac{f(v^*, \pi^*)}{1 - \beta}$
\hrule\vspace{0.5ex}
\Step iteration $\leftarrow 1$. 
\Step $\widehat{\pi} \leftarrow \pi_0$.
\Step Compute $\widehat{v}$ from linear programs in (\ref{eq:grad-tech:value-only}) using only the first phase of the revised simplex method (see Section \ref{sect:grad-tech:scheme:initial-point}).
\LOOP
\Step Compute feasible direction $S$ using the two-stage feasible direction method (algorithm~\ref{algo:grad-tech:twostagefeasible}).
\Step Stop algorithm if $S = 0$. $\left <v^*, \pi^*\right > \leftarrow \left <\widehat{v}, \widehat{\pi}\right >$. Output $\left <v^*, \pi^*\right >$ and $\epsilon = \dfrac{f(v^*, \pi^*)}{1 - \beta}$. Terminate the algorithm.
\Step Compute the constrained optimal step length $t$ by the procedure described in Section \ref{sect:grad-tech:scheme:optimal-step}.
\Step Stop algorithm if $t = 0$. $\left <v^*, \pi^*\right > \leftarrow \left <\widehat{v}, \widehat{\pi}\right >$. Output $\left <v^*, \pi^*\right >$ and $\epsilon = \dfrac{f(v^*, \pi^*)}{1 - \beta}$. Terminate the algorithm.
\Step $\left <\widehat{v}, \widehat{\pi}\right > \leftarrow \left <\widehat{v}, \widehat{\pi}\right > + t S$
\Step iteration $\leftarrow$ iteration $+~1$.
\ENDLOOP
\end{algorithmic}
\end{nonfloatalgorithm}

Note that in the above algorithm, equality of $S$ to zero and also that of $t$ to zero are to be considered with a small error bound around zero to handle numerical issues. The computational complexity per iteration of the algorithm is $O(|\A|^3)$ multiplications contributed mainly from the steps involving formation and decomposition of the inner product matrix, $G$. However, the factor multiplying $|\A|^3$ can be shown to be far less than one in the actual implementation. %Refer to \cite[Section 3.8]{sg-self-tr} for detailed computational complexity analysis.

\subsection{Convergence to a KKT point}
\label{subsect:conv-kkt}
KKT conditions represent a set of necessary and sufficient conditions for a point to be a valid local minimum of an optimization problem. We write down the necessary conditions for a point $\left < v^*, \pi^* \right >$ to be a local minimum of the optimization problem \eqref{eq:opt:2-player-opt}:
\begin{equation}
\label{eq:grad-tech:kkt}
\left .\begin{array}{l}
\subequationitem\label{subeq:grad-tech:kkt-1} \nabla f(v^*, \pi^*) + \sum_{j = 1}^n \lambda_j \nabla g_j(v^*, \pi^*) = 0,\\
\subequationitem\label{subeq:grad-tech:kkt-slack} \lambda_j g_j(v^*, \pi^*) = 0, j = 1, 2, \dots, n,\\
\subequationitem g_j(v^*, \pi^*) \ge 0, j = 1, 2, \dots, n,\\
\subequationitem \lambda_j \ge 0, j = 1, 2, \dots, n,
\end{array}\right \}
\end{equation}
where $\lambda_j, j = 1, 2, \dots, n$, are the Lagrange multipliers associated with the constraints, $g_j(\v, \pi) \ge 0, j = 1, 2, \dots, n$. Let $I = \{j| g_j(\v, \pi) = 0\}$ be the set of active constraints. It can be easily shown that, the above set of conditions are sufficient as well if, the gradients of all active constraints form a linearly independent set. 

We note here that the entire proof of convergence to KKT point presented in \cite[Section~3]{twostagefeasible} for the unmodified Herskovits algorithm, can easily be seen to be applicable as it is, to the modified Herskovits algorithm, i.e., Algorithm~\ref{algo:grad-tech:scheme}. In the next section, we apply this algorithm to a simple terrain exploration problem, modelled as a general-sum discounted stochastic game, and observe in the simulations that the convergence is also to a Nash equilibrium. However, later in Section~\ref{sect:grad-tech:fallacies}, we show that in general, convergence to a KKT point is not sufficient to guarantee convergence to a Nash equilibrium.

\section{A Simple Terrain Exploration Problem}
\label{sect:grad-tech:terrain}
A simplified version of the general terrain exploration problem is presented below. Consider a pair of agents that are assigned the task of collecting a set of objects located at various positions in a terrain. We assume that the object positions are known {\em aproiri}. The game between the pair of agents terminates if all the objects are collected. The agent movements are considered to be stochastic. Modelling of this problem as a discounted stochastic game $\left < \S, \A, p, r, \beta \right >$ is described as follows.
\begin{inparaenum}[\\\bfseries (i)]
\item \textbf{State Space, $\S$} - Let the entire terrain be discretized into a grid structure defined by $S_G = G \times G$ where $G = \left \{0, \pm 1, \pm 2, \dots, \pm M\right \}.$ The position of an agent can be represented by a point in $S_G$. Let the position of the $i^{th}$ agent be denoted by $x^i \in S_G$ with $x^{i(1)}, x^{i(2)} \in G$ being its two co-ordinate components. So, the positional part of the overall state space considering the two agents, is given by $\S_p = S_G \times S_G.$ The status regarding whether a particular object is collected or not is also a part of the state space. So, the overall state space would be given by $\S' = \S_p \times \left \{0,1\right \}^K,$ where $K$ is the total number of objects to be collected from the terrain. Let $o_i$ represent the Boolean variable for the status of the $i^{th}$ object. Here $o_i = 0$ implies that the $i^{th}$ object is not yet collected and the opposite is true for $o_i = 1$. Thus, $x = \left < x^1, x^2, o_1, o_2, \dots, o_K \right > \in \S'$ where $x^i \in S_G$, $i = 1,2$. Let $B = \left \{ y \in S_G: \text{an object is located at }y \right \}$. The two sets $\S_1 = \left \{ x \in \S':  x^i \in B\text{ and } o_{x^i} = 0\text{ for some }i = 1,2 \right \}$ and $\S_2 = \left \{ x \in \S': o_j = 1\quad\forall j = 1 \text{ to } K \right \},$ represent those combinations of states which are not feasible. Thus, the actual state space containing only feasible states is $\S = \left [ \S'\backslash \left ( \S_1 \cup \S_2 \right ) \right ] \cup T$, where $T$ represents the terminal state of the game.
\item \textbf{Action Space, $\A$} - The action space of the \ith{i} agent can be defined as \[\A^i(x) = \left \{ \text{ Go to } y \in S_G : d_\infty(x^i,y) \leq 1 \right \},\]
where $x^i \in S_G$ is the position of the \ith{i} agent and $d_\infty(x^i,y) = max(|x^{i(1)} - y^{(1)}|,|x^{i(2)} - y^{(2)}|)$ is the $L^{\infty}$ distance metric. The aggregate action space of the two agents at state $x \in \S\backslash \{T\}$ is given by $\A(x) = \A^1(x) \times \A^2(x).$ Note that $x = \left <x^1, x^2, o_1, o_2,\dots, o_K\right >$. Thus, the action space does not depend upon the object state except for the termination state $T$. For the termination state $T$, the only action available is to stay in the termination state. The action related to the termination state $T$ is ignored in subsequent discussions.
\item \textbf{Transition Probability, $p(y|x,a)$} - The movements of each agent are assumed to be independent of other agents. The transition probability $p^i(y^i|x^i, a^i)$ for the \ith{i} agent is given by $p^i(y^i|x^i,a^i) = C(x^i)\ 2^{-d_1(a^i,y^i)}\ \forall y^i \in U^i(x^i) \subseteq S_G, i = 1,2,$ where $C(x^i) = \sum \limits_{y \in U^i(x^i)} 2^{-d_1(a^i,y)}$ is the normalization factor chosen to make this a probability measure and $d_1(a^i,y) = \left ( | a^{i(1)} - y^{(1)} | + | a^{i(2)} - y^{(2)} | \right ),$ the $L^1$ norm distance between $a^i$ and $y$. The joint transition probability is given by $p(y|x,a) = p^1(y^1|x^1,a^1) p^2(y^2|x^2,a^2)$.
\item \textbf{Reward function, $r(x,a)$} - To ensure that the two agents do not get to the same position, a penalty may be imposed on the two agents when they attain the same position. Thus, the stochastic reward function for the $i^{th}$ agent can be defined accordingly as
\begin{equation}
\overline{r}^i(x,a,y) = 
\left \{\begin{array}{ll}
-\frac{1}{2} & \text{if } y^i = y^j, j = 1, 2, j \neq i \text{ and } O_y \neq \left < 1, 1, \dots, 1 \right >, \\
1 & \text{if object present at $y^i$,}\\
0 & \text{otherwise,}
\end{array}\right .
\end{equation}
$i = 1, 2.$ The reward $r^{i}(x,a)$ is given by $r^i(x,a) = \sum \limits_{y \in \S} \overline{r}^i(x,a,y) p(y|x,a)$.
%\item \textbf{Discount Factor, $\beta$} takes values in the open interval $(0,1)$.
\end{inparaenum}

\subsection{Simulation Results}
Simulation results for $G = \left \{0, 1, 2, 3\right \}$ with two objects situated at $(0,3)$ and $(3,3)$ and discount factor $\beta = 0.75$ are described below. The parameters given to the two-stage feasible direction method are $w_j(v_0, \pi_0) = 1, j = 1,2,\dots, n$, $\alpha = 0.5$ and $\rho_0 = 0.9$.

\subsubsection{Objective Value}
The convergence of the objective value using Algorithm~\ref{algo:grad-tech:scheme}, to a value close to zero is shown in Figure \ref{fig:objective-value}. After getting an initial feasible solution, the objective value was $\approx 102.37$.

%The objective value evolution with iterations illustrates a peculiar behavior of our algorithm. Only alternate iterations registered a significant change in the objective value. So, the plot shown in Figure \ref{fig:objective-value} has a straggered decrease in value with number of iterations as highlighted in the inset of the figure.

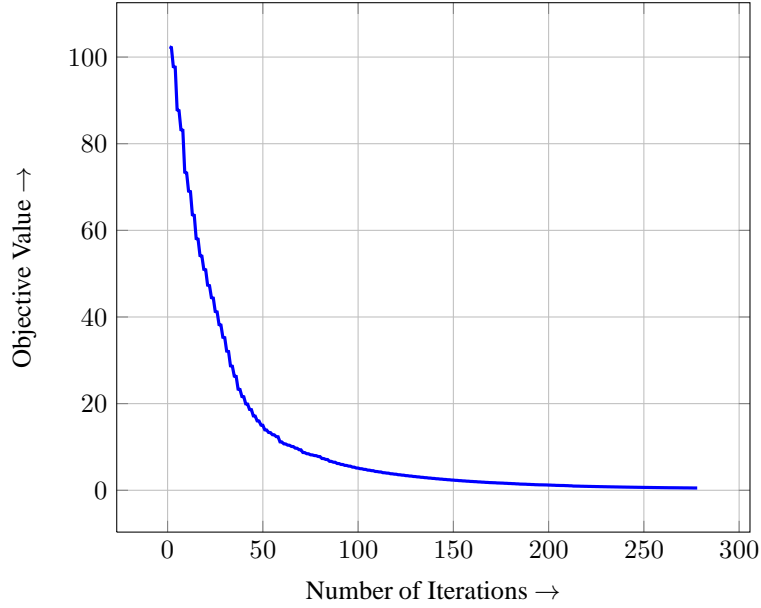
\begin{figure}
\centering
\begin{tikzpicture}

\begin{axis}[xlabel={Number of Iterations $\rightarrow$}, ylabel={Objective Value $\rightarrow$},width=10cm,grid,no markers,tick scale binop={\times}]
\addplot[color=blue,very thick] file {plots/objective4by4-2.data};
\end{axis}

% \begin{scope}[yshift=-10cm]
% \begin{scope}[xscale=0.05mm,yscale=0.1mm]
% \clip (20,50) circle (5);
% \draw (20,50) circle (5);
% \draw[thick] plot file {objective4by4-2.data};
% \end{scope}
% \end{scope}
% 
% \begin{scope}[xscale=0.01mm,yscale=0.02mm]
% \draw[->,thick] (80,60) -- (20,50);
% \end{scope}

\end{tikzpicture}
\caption{Objective Value vs. Number of Iterations}
\label{fig:objective-value}
\end{figure}

\subsubsection{Strategies}
The convergence behaviour of strategies of both agents with the initial position of the first agent being (2,1) and that of the second being (2,0), respectively, is shown in Figures \ref{fig:strategy-agent1} and \ref{fig:strategy-agent2} respectively. The arrows in the various grids in Figures \ref{fig:strategy-agent1} and \ref{fig:strategy-agent2} signify the feasible actions in each state and their lengths are proportional to the transition probabilities along the corresponding directions. With the given initial positions of agents and object locations, strategies pertaining only to those positions which an agent can visit with the other agent sticking to its own position are plotted. Consider for instance, Figure \ref{fig:strategy-agent1}. The figure shows the strategy of the first agent with the second agent sticking to the position (2,0). At the start of the algorithm, all transition probabilities are chosen according to the uniform distribution. In Figures \ref{fig:strategy-agent1} and \ref{fig:strategy-agent2}, we show the strategy profile of both the agents  after the $1^{\text{st}}$, \ith{11} and \ith{100} iterations, and upon convergence of the algorithm. The algorithm converged in a total of 278 iterations.

The Nash strategies have an interesting structure here which is evident in Figures \ref{fig:strategy-agent1} and \ref{fig:strategy-agent2} as well. The strategies are  deterministic except when both agents are in the vicinity of one another. This is expected from the structure of the reward function. Also, it is clear from Figures \ref{fig:strategy-agent1} and \ref{fig:strategy-agent2} that strategy components that are near to the two objects converge faster compared to those which are farther from the two objects. Note that strategy components of those positions which have no probability of being visited by an agent, with the agent being in a particular position, are not shown with arrow marks. For instance, in Figure \ref{fig:strategy-agent1}\subref{subfig:agent1-onconvergence}, position (1,1) has no probability of being visited by the first agent located at (2,1). 

\begin{figure}[h]
\begin{tabular}{cc}
\begin{nonfloatfigure}[0.48\textwidth]
\begin{tabular}{cc}
\subfigure[After First Iteration]{
\begin{tikzpicture}[scale=0.6,font=\footnotesize]
\draw[xshift=-0.5cm,yshift=-0.5cm] (0,0) grid (4,4);

\foreach \x in {0,1,2,3} {
	\draw (\x,-0.7) node {$\x$};
	\draw (-0.7,\x) node {$\x$};
}

% Objects
\draw[fill=black] (0,3) circle(2pt);
\draw[fill=black] (3,3) circle(2pt);

%Agents
\draw (2,0) node[fill=black,rectangle,inner sep=2pt] {};
\draw (2,1) node[fill=black,rectangle,inner sep=2pt] {};

%Strategy of first agent
\draw[->] (2, 1) -- +(225:0.347967);
\draw[->] (2, 1) -- +(180:0.407429);
\draw[->] (2, 1) -- +(135:0.5);
\draw[->] (2, 1) -- +(-90:0.335074);
\draw[->] (2, 1) -- +(90:0.491159);
\draw[->] (2, 1) -- +(-45:0.329515);
\draw[->] (2, 1) -- +(0:0.390542);
\draw[->] (2, 1) -- +(45:0.498495);
\draw[->] (1, 0) -- +(180:0.466296);
\draw[->] (1, 0) -- +(135:0.5);
\draw[->] (1, 0) -- +(90:0.462815);
\draw[->] (1, 0) -- +(0:0.412308);
\draw[->] (1, 0) -- +(45:0.448426);
\draw[->] (0, 0) -- +(90:0.5);
\draw[->] (0, 0) -- +(0:0.432251);
\draw[->] (0, 0) -- +(45:0.465645);
\draw[->] (0, 1) -- +(-90:0.374325);
\draw[->] (0, 1) -- +(90:0.5);
\draw[->] (0, 1) -- +(-45:0.342983);
\draw[->] (0, 1) -- +(0:0.397745);
\draw[->] (0, 1) -- +(45:0.49085);
\draw[->] (0, 2) -- +(-90:0.310645);
\draw[->] (0, 2) -- +(90:0.5);
\draw[->] (0, 2) -- +(-45:0.276008);
\draw[->] (0, 2) -- +(0:0.356931);
\draw[->] (0, 2) -- +(45:0.434366);
\draw[->] (1, 1) -- +(225:0.369609);
\draw[->] (1, 1) -- +(180:0.421808);
\draw[->] (1, 1) -- +(135:0.5);
\draw[->] (1, 1) -- +(-90:0.335077);
\draw[->] (1, 1) -- +(90:0.478285);
\draw[->] (1, 1) -- +(-45:0.318684);
\draw[->] (1, 1) -- +(0:0.376409);
\draw[->] (1, 1) -- +(45:0.474066);
\draw[->] (1, 2) -- +(225:0.318669);
\draw[->] (1, 2) -- +(180:0.404936);
\draw[->] (1, 2) -- +(135:0.5);
\draw[->] (1, 2) -- +(-90:0.282224);
\draw[->] (1, 2) -- +(90:0.426556);
\draw[->] (1, 2) -- +(-45:0.268526);
\draw[->] (1, 2) -- +(0:0.349024);
\draw[->] (1, 2) -- +(45:0.4043);
\draw[->] (1, 3) -- +(225:0.417581);
\draw[->] (1, 3) -- +(180:0.5);
\draw[->] (1, 3) -- +(-90:0.372611);
\draw[->] (1, 3) -- +(-45:0.357946);
\draw[->] (1, 3) -- +(0:0.385777);
\draw[->] (2, 2) -- +(225:0.287229);
\draw[->] (2, 2) -- +(180:0.361985);
\draw[->] (2, 2) -- +(135:0.411564);
\draw[->] (2, 2) -- +(-90:0.275597);
\draw[->] (2, 2) -- +(90:0.424944);
\draw[->] (2, 2) -- +(-45:0.283508);
\draw[->] (2, 2) -- +(0:0.385743);
\draw[->] (2, 2) -- +(45:0.5);
\draw[->] (2, 3) -- +(225:0.355495);
\draw[->] (2, 3) -- +(180:0.385977);
\draw[->] (2, 3) -- +(-90:0.363493);
\draw[->] (2, 3) -- +(-45:0.40343);
\draw[->] (2, 3) -- +(0:0.5);
\draw[->] (3, 2) -- +(225:0.260622);
\draw[->] (3, 2) -- +(180:0.346107);
\draw[->] (3, 2) -- +(135:0.416354);
\draw[->] (3, 2) -- +(-90:0.27393);
\draw[->] (3, 2) -- +(90:0.5);
\draw[->] (3, 1) -- +(225:0.335858);
\draw[->] (3, 1) -- +(180:0.394933);
\draw[->] (3, 1) -- +(135:0.486458);
\draw[->] (3, 1) -- +(-90:0.332553);
\draw[->] (3, 1) -- +(90:0.5);
\draw[->] (2, 0) -- +(180:0.459122);
\draw[->] (2, 0) -- +(135:0.5);
\draw[->] (2, 0) -- +(90:0.481219);
\draw[->] (2, 0) -- +(0:0.432579);
\draw[->] (2, 0) -- +(45:0.4728);
\draw[->] (3, 0) -- +(180:0.460792);
\draw[->] (3, 0) -- +(135:0.5);
\draw[->] (3, 0) -- +(90:0.496344);
\end{tikzpicture}} &

\subfigure[After \ith{11} Iteration]{
\begin{tikzpicture}[scale=0.6,font=\footnotesize]
\draw[xshift=-0.5cm,yshift=-0.5cm] (0,0) grid (4,4);

\foreach \x in {0,1,2,3} {
	\draw (\x,-0.7) node {$\x$};
	\draw (-0.7,\x) node {$\x$};
}

% Objects
\draw[fill=black] (0,3) circle(2pt);
\draw[fill=black] (3,3) circle(2pt);

%Agents
\draw (2,0) node[fill=black,rectangle,inner sep=2pt] {};
\draw (2,1) node[fill=black,rectangle,inner sep=2pt] {};

%Strategies
\draw[->] (2, 1) -- +(225:0.098497);
\draw[->] (2, 1) -- +(180:0.202918);
\draw[->] (2, 1) -- +(135:0.500000);
\draw[->] (2, 1) -- +(-90:0.081615);
\draw[->] (2, 1) -- +(90:0.459457);
\draw[->] (2, 1) -- +(-45:0.074571);
\draw[->] (2, 1) -- +(0:0.162614);
\draw[->] (2, 1) -- +(45:0.487879);
\draw[->] (1, 0) -- +(180:0.372636);
\draw[->] (1, 0) -- +(135:0.500000);
\draw[->] (1, 0) -- +(90:0.364697);
\draw[->] (1, 0) -- +(0:0.221302);
\draw[->] (1, 0) -- +(45:0.316543);
\draw[->] (0, 0) -- +(90:0.500000);
\draw[->] (0, 0) -- +(0:0.273767);
\draw[->] (0, 0) -- +(45:0.376103);
\draw[->] (0, 1) -- +(-90:0.137510);
\draw[->] (0, 1) -- +(90:0.500000);
\draw[->] (0, 1) -- +(-45:0.093927);
\draw[->] (0, 1) -- +(0:0.184669);
\draw[->] (0, 1) -- +(45:0.466818);
\draw[->] (0, 2) -- +(-90:0.054566);
\draw[->] (0, 2) -- +(90:0.500000);
\draw[->] (0, 2) -- +(0:0.108196);
\draw[->] (0, 2) -- +(45:0.262560);
\draw[->] (1, 2) -- +(225:0.063028);
\draw[->] (1, 2) -- +(180:0.195194);
\draw[->] (1, 2) -- +(135:0.500000);
\draw[->] (1, 2) -- +(90:0.248686);
\draw[->] (1, 2) -- +(0:0.098087);
\draw[->] (1, 2) -- +(45:0.190014);
\draw[->] (1, 3) -- +(225:0.228749);
\draw[->] (1, 3) -- +(180:0.500000);
\draw[->] (1, 3) -- +(-90:0.135354);
\draw[->] (1, 3) -- +(-45:0.111968);
\draw[->] (1, 3) -- +(0:0.157143);
\draw[->] (2, 2) -- +(180:0.109831);
\draw[->] (2, 2) -- +(135:0.200383);
\draw[->] (2, 2) -- +(90:0.231993);
\draw[->] (2, 2) -- +(0:0.148074);
\draw[->] (2, 2) -- +(45:0.500000);
\draw[->] (2, 3) -- +(225:0.104248);
\draw[->] (2, 3) -- +(180:0.153247);
\draw[->] (2, 3) -- +(-90:0.115888);
\draw[->] (2, 3) -- +(-45:0.188861);
\draw[->] (2, 3) -- +(0:0.500000);
\draw[->] (3, 2) -- +(180:0.086906);
\draw[->] (3, 2) -- +(135:0.213315);
\draw[->] (3, 2) -- +(90:0.500000);
\draw[->] (1, 1) -- +(225:0.129038);
\draw[->] (1, 1) -- +(180:0.235004);
\draw[->] (1, 1) -- +(135:0.500000);
\draw[->] (1, 1) -- +(-90:0.084385);
\draw[->] (1, 1) -- +(90:0.417070);
\draw[->] (1, 1) -- +(-45:0.067501);
\draw[->] (1, 1) -- +(0:0.143694);
\draw[->] (1, 1) -- +(45:0.402095);
\draw[->] (2, 0) -- +(180:0.348146);
\draw[->] (2, 0) -- +(135:0.500000);
\draw[->] (2, 0) -- +(90:0.426587);
\draw[->] (2, 0) -- +(0:0.271484);
\draw[->] (2, 0) -- +(45:0.395220);
\draw[->] (3, 0) -- +(180:0.356108);
\draw[->] (3, 0) -- +(135:0.500000);
\draw[->] (3, 0) -- +(90:0.490576);
\draw[->] (3, 1) -- +(225:0.081741);
\draw[->] (3, 1) -- +(180:0.173299);
\draw[->] (3, 1) -- +(135:0.439474);
\draw[->] (3, 1) -- +(-90:0.078976);
\draw[->] (3, 1) -- +(90:0.500000);
\end{tikzpicture}}\\

\subfigure[After \ith{100} Iteration]{
\begin{tikzpicture}[scale=0.6,font=\footnotesize]
\draw[xshift=-0.5cm,yshift=-0.5cm] (0,0) grid (4,4);

\foreach \x in {0,1,2,3} {
	\draw (\x,-0.7) node {$\x$};
	\draw (-0.7,\x) node {$\x$};
}

% Objects
\draw[fill=black] (0,3)	 circle(2pt);
\draw[fill=black] (3,3) circle(2pt);

%Agents
\draw (2,0) node[fill=black,rectangle,inner sep=2pt] {};
\draw (2,1) node[fill=black,rectangle,inner sep=2pt] {};

%Strategies
\draw[->] (2, 1) -- +(135:0.500000);
\draw[->] (2, 1) -- +(90:0.201638);
\draw[->] (2, 1) -- +(45:0.144380);
\draw[->] (1, 2) -- +(135:0.500000);
\draw[->] (2, 2) -- +(45:0.500000);
\draw[->] (3, 2) -- +(90:0.500000);
\end{tikzpicture}} &

\subfigure[On Convergence]{\label{subfig:agent1-onconvergence}
\begin{tikzpicture}[scale=0.6,font=\footnotesize]
\draw[xshift=-0.5cm,yshift=-0.5cm] (0,0) grid (4,4);

\foreach \x in {0,1,2,3} {
	\draw (\x,-0.7) node {$\x$};
	\draw (-0.7,\x) node {$\x$};
}

% Objects
\draw[fill=black] (0,3) circle(2pt);
\draw[fill=black] (3,3) circle(2pt);

%Agents
\draw (2,0) node[fill=black,rectangle,inner sep=2pt] {};
\draw (2,1) node[fill=black,rectangle,inner sep=2pt] {};

%Strategies
\draw[->] (2, 1) -- +(135:0.500000);
\draw[->] (2, 1) -- +(90:0.288682);
\draw[->] (2, 1) -- +(45:0.095299);
\draw[->] (1, 2) -- +(135:0.500000);
\draw[->] (2, 2) -- +(45:0.500000);
\draw[->] (3, 2) -- +(90:0.500000);
\end{tikzpicture}}
\end{tabular}

\caption{Convergence of the strategy updates of the first agent when it is located at (2,1) and the other agent is located at (2,0).}
\label{fig:strategy-agent1}
\end{nonfloatfigure}
& 
\begin{nonfloatfigure}[0.48\textwidth]
\begin{tabular}{cc}
\subfigure[After First Iteration]{
\begin{tikzpicture}[scale=0.6,font=\footnotesize]
\draw[xshift=-0.5cm,yshift=-0.5cm] (0,0) grid (4,4);

\foreach \x in {0,1,2,3} {
	\draw (\x,-0.7) node {$\x$};
	\draw (-0.7,\x) node {$\x$};
}

% Objects
\draw[fill=black] (0,3) circle(2pt);
\draw[fill=black] (3,3) circle(2pt);

%Agents
\draw (2,0) node[fill=black,rectangle,inner sep=2pt] {};
\draw (2,1) node[fill=black,rectangle,inner sep=2pt] {};

%Strategies
\draw[->] (2, 0) -- +(180:0.469659);
\draw[->] (2, 0) -- +(135:0.500000);
\draw[->] (2, 0) -- +(90:0.487390);
\draw[->] (2, 0) -- +(0:0.451465);
\draw[->] (2, 0) -- +(45:0.481921);
\draw[->] (1, 0) -- +(180:0.471288);
\draw[->] (1, 0) -- +(135:0.500000);
\draw[->] (1, 0) -- +(90:0.470687);
\draw[->] (1, 0) -- +(0:0.428630);
\draw[->] (1, 0) -- +(45:0.456432);
\draw[->] (0, 0) -- +(90:0.500000);
\draw[->] (0, 0) -- +(0:0.445820);
\draw[->] (0, 0) -- +(45:0.473878);
\draw[->] (0, 1) -- +(-90:0.393028);
\draw[->] (0, 1) -- +(90:0.500000);
\draw[->] (0, 1) -- +(-45:0.367049);
\draw[->] (0, 1) -- +(0:0.410663);
\draw[->] (0, 1) -- +(45:0.480511);
\draw[->] (0, 2) -- +(-90:0.317794);
\draw[->] (0, 2) -- +(90:0.500000);
\draw[->] (0, 2) -- +(-45:0.281756);
\draw[->] (0, 2) -- +(0:0.349590);
\draw[->] (0, 2) -- +(45:0.430250);
\draw[->] (1, 1) -- +(225:0.395045);
\draw[->] (1, 1) -- +(180:0.438277);
\draw[->] (1, 1) -- +(135:0.500000);
\draw[->] (1, 1) -- +(-90:0.366148);
\draw[->] (1, 1) -- +(90:0.469791);
\draw[->] (1, 1) -- +(-45:0.351683);
\draw[->] (1, 1) -- +(0:0.394070);
\draw[->] (1, 1) -- +(45:0.459258);
\draw[->] (1, 2) -- +(225:0.329097);
\draw[->] (1, 2) -- +(180:0.403952);
\draw[->] (1, 2) -- +(135:0.500000);
\draw[->] (1, 2) -- +(-90:0.291518);
\draw[->] (1, 2) -- +(90:0.423000);
\draw[->] (1, 2) -- +(-45:0.276749);
\draw[->] (1, 2) -- +(0:0.339004);
\draw[->] (1, 2) -- +(45:0.398405);
\draw[->] (1, 3) -- +(225:0.413102);
\draw[->] (1, 3) -- +(180:0.500000);
\draw[->] (1, 3) -- +(-90:0.359056);
\draw[->] (1, 3) -- +(-45:0.343387);
\draw[->] (1, 3) -- +(0:0.382357);
\draw[->] (2, 2) -- +(225:0.304229);
\draw[->] (2, 2) -- +(180:0.368915);
\draw[->] (2, 2) -- +(135:0.422658);
\draw[->] (2, 2) -- +(-90:0.292310);
\draw[->] (2, 2) -- +(90:0.429042);
\draw[->] (2, 2) -- +(-45:0.299969);
\draw[->] (2, 2) -- +(0:0.380422);
\draw[->] (2, 2) -- +(45:0.500000);
\draw[->] (2, 1) -- +(225:0.391903);
\draw[->] (2, 1) -- +(180:0.435884);
\draw[->] (2, 1) -- +(135:0.500000);
\draw[->] (2, 1) -- +(-90:0.379798);
\draw[->] (2, 1) -- +(90:0.485271);
\draw[->] (2, 1) -- +(-45:0.373946);
\draw[->] (2, 1) -- +(0:0.417049);
\draw[->] (2, 1) -- +(45:0.486270);
\draw[->] (3, 0) -- +(180:0.471378);
\draw[->] (3, 0) -- +(135:0.500000);
\draw[->] (3, 0) -- +(90:0.497893);
\draw[->] (3, 1) -- +(225:0.384865);
\draw[->] (3, 1) -- +(180:0.428415);
\draw[->] (3, 1) -- +(135:0.492329);
\draw[->] (3, 1) -- +(-90:0.381412);
\draw[->] (3, 1) -- +(90:0.500000);
\draw[->] (3, 2) -- +(225:0.273920);
\draw[->] (3, 2) -- +(180:0.342042);
\draw[->] (3, 2) -- +(135:0.417490);
\draw[->] (3, 2) -- +(-90:0.287127);
\draw[->] (3, 2) -- +(90:0.500000);
\draw[->] (2, 3) -- +(225:0.371456);
\draw[->] (2, 3) -- +(180:0.404734);
\draw[->] (2, 3) -- +(-90:0.360398);
\draw[->] (2, 3) -- +(-45:0.390395);
\draw[->] (2, 3) -- +(0:0.500000);
\end{tikzpicture}}
&
\subfigure[After \ith{11} Iteration]{
\begin{tikzpicture}[scale=0.6,font=\footnotesize]
\draw[xshift=-0.5cm,yshift=-0.5cm] (0,0) grid (4,4);

\foreach \x in {0,1,2,3} {
	\draw (\x,-0.7) node {$\x$};
	\draw (-0.7,\x) node {$\x$};
}

% Objects
\draw[fill=black] (0,3) circle(2pt);
\draw[fill=black] (3,3) circle(2pt);

%Agents
\draw (2,0) node[fill=black,rectangle,inner sep=2pt] {};
\draw (2,1) node[fill=black,rectangle,inner sep=2pt] {};

%Strategies
\draw[->] (2, 0) -- +(180:0.394469);
\draw[->] (2, 0) -- +(135:0.500000);
\draw[->] (2, 0) -- +(90:0.454487);
\draw[->] (2, 0) -- +(0:0.340562);
\draw[->] (2, 0) -- +(45:0.435178);
\draw[->] (1, 0) -- +(180:0.394016);
\draw[->] (1, 0) -- +(135:0.500000);
\draw[->] (1, 0) -- +(90:0.395633);
\draw[->] (1, 0) -- +(0:0.272818);
\draw[->] (1, 0) -- +(45:0.347733);
\draw[->] (0, 0) -- +(90:0.500000);
\draw[->] (0, 0) -- +(0:0.318467);
\draw[->] (0, 0) -- +(45:0.407258);
\draw[->] (0, 1) -- +(-90:0.177499);
\draw[->] (0, 1) -- +(90:0.500000);
\draw[->] (0, 1) -- +(-45:0.134497);
\draw[->] (0, 1) -- +(0:0.217797);
\draw[->] (0, 1) -- +(45:0.423887);
\draw[->] (0, 2) -- +(-90:0.062483);
\draw[->] (0, 2) -- +(90:0.500000);
\draw[->] (0, 2) -- +(0:0.098425);
\draw[->] (0, 2) -- +(45:0.251479);
\draw[->] (1, 2) -- +(225:0.074313);
\draw[->] (1, 2) -- +(180:0.193422);
\draw[->] (1, 2) -- +(135:0.500000);
\draw[->] (1, 2) -- +(90:0.239639);
\draw[->] (1, 2) -- +(0:0.085824);
\draw[->] (1, 2) -- +(45:0.177794);
\draw[->] (1, 3) -- +(225:0.217660);
\draw[->] (1, 3) -- +(180:0.500000);
\draw[->] (1, 3) -- +(-90:0.113341);
\draw[->] (1, 3) -- +(-45:0.091397);
\draw[->] (1, 3) -- +(0:0.149931);
\draw[->] (2, 2) -- +(180:0.123500);
\draw[->] (2, 2) -- +(135:0.229741);
\draw[->] (2, 2) -- +(90:0.242253);
\draw[->] (2, 2) -- +(0:0.137958);
\draw[->] (2, 2) -- +(45:0.500000);
\draw[->] (2, 3) -- +(225:0.129437);
\draw[->] (2, 3) -- +(180:0.192311);
\draw[->] (2, 3) -- +(-90:0.110324);
\draw[->] (2, 3) -- +(-45:0.160586);
\draw[->] (2, 3) -- +(0:0.500000);
\draw[->] (3, 2) -- +(180:0.081276);
\draw[->] (3, 2) -- +(135:0.213849);
\draw[->] (3, 2) -- +(90:0.500000);
\draw[->] (1, 1) -- +(225:0.180701);
\draw[->] (1, 1) -- +(180:0.283138);
\draw[->] (1, 1) -- +(135:0.500000);
\draw[->] (1, 1) -- +(-90:0.132966);
\draw[->] (1, 1) -- +(90:0.386728);
\draw[->] (1, 1) -- +(-45:0.112394);
\draw[->] (1, 1) -- +(0:0.181853);
\draw[->] (1, 1) -- +(45:0.350118);
\draw[->] (2, 1) -- +(225:0.175269);
\draw[->] (2, 1) -- +(180:0.276010);
\draw[->] (2, 1) -- +(135:0.500000);
\draw[->] (2, 1) -- +(-90:0.152260);
\draw[->] (2, 1) -- +(90:0.434016);
\draw[->] (2, 1) -- +(-45:0.141468);
\draw[->] (2, 1) -- +(0:0.223754);
\draw[->] (2, 1) -- +(45:0.432956);
\draw[->] (3, 0) -- +(180:0.402993);
\draw[->] (3, 0) -- +(135:0.500000);
\draw[->] (3, 0) -- +(90:0.490961);
\draw[->] (3, 1) -- +(225:0.170986);
\draw[->] (3, 1) -- +(180:0.264798);
\draw[->] (3, 1) -- +(135:0.469025);
\draw[->] (3, 1) -- +(-90:0.164468);
\draw[->] (3, 1) -- +(90:0.500000);
\end{tikzpicture}}\\
\subfigure[After \ith{100} Iteration]{
\begin{tikzpicture}[scale=0.6,font=\footnotesize]
\draw[xshift=-0.5cm,yshift=-0.5cm] (0,0) grid (4,4);

\foreach \x in {0,1,2,3} {
	\draw (\x,-0.7) node {$\x$};
	\draw (-0.7,\x) node {$\x$};
}

% Objects
\draw[fill=black] (0,3) circle(2pt);
\draw[fill=black] (3,3) circle(2pt);

%Agents
\draw (2,0) node[fill=black,rectangle,inner sep=2pt] {};
\draw (2,1) node[fill=black,rectangle,inner sep=2pt] {};

%Strategies
\draw[->] (2, 0) -- +(180:0.120312);
\draw[->] (2, 0) -- +(135:0.500000);
\draw[->] (2, 0) -- +(90:0.410680);
\draw[->] (2, 0) -- +(0:0.095422);
\draw[->] (2, 0) -- +(45:0.463198);
\draw[->] (1, 0) -- +(180:0.091659);
\draw[->] (1, 0) -- +(135:0.500000);
\draw[->] (1, 0) -- +(90:0.156055);
\draw[->] (0, 0) -- +(90:0.500000);
\draw[->] (0, 0) -- +(45:0.199418);
\draw[->] (0, 1) -- +(90:0.500000);
\draw[->] (0, 1) -- +(45:0.113069);
\draw[->] (0, 2) -- +(90:0.500000);
\draw[->] (1, 2) -- +(135:0.500000);
\draw[->] (1, 1) -- +(135:0.500000);
\draw[->] (1, 1) -- +(90:0.072345);
\draw[->] (2, 1) -- +(135:0.500000);
\draw[->] (2, 1) -- +(90:0.232972);
\draw[->] (2, 1) -- +(45:0.285958);
\draw[->] (2, 2) -- +(45:0.500000);
\draw[->] (3, 2) -- +(90:0.500000);
\draw[->] (3, 0) -- +(180:0.141185);
\draw[->] (3, 0) -- +(135:0.500000);
\draw[->] (3, 0) -- +(90:0.442526);
\draw[->] (3, 1) -- +(135:0.285035);
\draw[->] (3, 1) -- +(90:0.500000);
\end{tikzpicture}}
&
\subfigure[On Convergence]{
\begin{tikzpicture}[scale=0.6,font=\footnotesize]
\draw[xshift=-0.5cm,yshift=-0.5cm] (0,0) grid (4,4);

\foreach \x in {0,1,2,3} {
	\draw (\x,-0.7) node {$\x$};
	\draw (-0.7,\x) node {$\x$};
}

% Objects
\draw[fill=black] (0,3) circle(2pt);
\draw[fill=black] (3,3) circle(2pt);

%Agents
\draw (2,0) node[fill=black,rectangle,inner sep=2pt] {};
\draw (2,1) node[fill=black,rectangle,inner sep=2pt] {};

%Strategies
\draw[->] (2, 0) -- +(135:0.500000);
\draw[->] (2, 0) -- +(90:0.214171);
\draw[->] (2, 0) -- +(45:0.439084);
\draw[->] (1, 1) -- +(135:0.500000);
\draw[->] (0, 2) -- +(90:0.500000);
\draw[->] (2, 1) -- +(135:0.500000);
\draw[->] (2, 1) -- +(90:0.072156);
\draw[->] (2, 1) -- +(45:0.378276);
\draw[->] (1, 2) -- +(135:0.500000);
\draw[->] (2, 2) -- +(45:0.500000);
\draw[->] (3, 2) -- +(90:0.500000);
\draw[->] (3, 1) -- +(90:0.500000);
\end{tikzpicture}}
\end{tabular}

\caption{Convergence of the strategy updates of the second agent when it is located at (2,0) and the other agent is located at (2,1).}
\label{fig:strategy-agent2}
\end{nonfloatfigure}
\end{tabular}
\end{figure}
%\FloatBarrier

\section{Non-Convergence to a Nash Equilibrium}
\label{sect:grad-tech:fallacies}
Theorem \ref{theorem:opt:fv2} showed that it is both necessary and sufficient for a feasible point $\left <\v^*, \pi^* \right >$ to correspond to a Nash equilibrium, if the objective value, $f(\v^*, \pi^*) = 0$. However, for a gradient-based scheme, it would be apt to have conditions represented in terms of gradients of the objective and constraints. In this direction, we now present a series of results which ultimately give the desired set of necessary and sufficient conditions for a minimum point to be a global minimum. For a given point $\left <\v, \pi \right >$, let $G = [ \nabla g_j(\v, \pi) : j = 1, 2, \dots, N ]$ represent a matrix whose columns are gradients of all the constraints \eqref{eq:grad-tech:2-player-opt-inequality}.

\begin{proposition}
\label{propos:grad-tech:objective-gradient}
At any given point $\left <\v, \pi \right >$, the gradient of the objective function $f(\v, \pi)$, can be expressed as a linear combination of the gradient of all the constraints \eqref{eq:grad-tech:2-player-opt-inequality}. In other words, $\nabla f(\v, \pi) = G \lambda'$ where $\lambda'$ is an appropriate vector.
\end{proposition}
\begin{proof}
Let $h_1(x, a^1) = \pi^2(x)^T \left [ \r^1(x, a^1, \A^2(x)) + \beta \sum \limits_{y \in U(x)} P(y|x, a^1, \A^2(x)) v^1(y) \right ] - v^1(x)$. Then, $h_1(x, a^1) \le 0$ represents the set of constraints \eqref{subeq:opt:2-player:agent1}. Similarly, let $h_2(x, a^2) = \left [ \r^2(x, \A^1(x), a^2) + \beta \sum \limits_{y \in U(x)} P(y|x, \A^1(x), a^2) v^2(y) \right ] \pi^1(x) - v^2(x)$. Thus, $h_2(x, a^2) \le 0$ represents the set of constraints \eqref{subeq:opt:2-player:agent2}. Now, we observe that the objective of the optimization problem \eqref{eq:opt:2-player-opt} can be re-expressed in terms of $h_i(x, a^i)$ and $\pi^i(x, a^i), i = 1, 2$, as follows:
\begin{equation}
\label{eq:grad-tech:objective-product-sum}
f(\v, \pi) = \sum_{i = 1}^2 \sum_{x \in \S} -\pi^i(x, a^i) h_i(x, a^i).
\end{equation}
So, the objective can be visualized as the sum of products between LHS of \eqref{subeq:opt:2-player:agent1} and that of corresponding constraints in  \eqref{subeq:opt:2-player:agent1prob}. Note that the equality constraints are easily eliminated as expressed in \eqref{eq:grad-tech:2-player-opt-inequality}. Thus, all the constraints of interest are inequality constraints which pair up, one from \eqref{subeq:opt:2-player:agent1}-\eqref{subeq:opt:2-player:agent2} and the other from \eqref{subeq:opt:2-player:agent1prob}-\eqref{subeq:opt:2-player:agent2prob}. It is now easy to see the desired result by considering the chain-rule of differentiation.
\end{proof}

Note that the vector $\lambda'$ discussed in the proposition \ref{propos:grad-tech:objective-gradient}, is in value the same as the negative of the pair constraint. For instance, let for some $j$, $g_j(\v, \pi) = h^i(x, a^i)$. Then, $\lambda_j' = -\pi^i(x, a^i)$. Similarly, for some $j$ for which $g_j(\v, \pi) = \pi^i(x, a^i)$, we have $\lambda_j' = -h^i(x, a^i)$. 

Let $\lambda' = \left [ \begin{array}{c} \lambda_I' \\ \lambda_K' \end{array} \right ]$, where $\lambda_I'$ is the part of $\lambda'$ corresponding to active constraints and $\lambda_K'$ is that corresponding to inactive constraints. Similarly, let the set of Lagrange multipliers, $\lambda = \left [ \begin{array}{c} \lambda_I \\ \lambda_K \end{array} \right ]$, where $\lambda_I$ is the part of $\lambda$ corresponding to active constraints and $\lambda_K$ is that corresponding to inactive constraints.

\begin{lemma}
Under assumption (\ref{asm:grad-tech:active-set-gradients}) of Section~\ref{sect:grad-tech:herskovits}, if $\lambda_K' = 0$ at a KKT point $\left <\v^*, \pi^* \right >$, then $\lambda_I = -\lambda_I'$.
\end{lemma}
\begin{proof}
Let $G = \left [ G_I \quad G_K \right ]$ be the previously defined matrix of gradients of all constraints, where $G_I$ is the part of matrix $G$ containing gradients of all active constraints and $G_K$ that containing gradients of all inactive constraints. Now, from proposition \ref{propos:grad-tech:objective-gradient}, we have that $\nabla f(\v^*, \pi^*) = G \lambda'$. From the KKT conditions \eqref{eq:grad-tech:kkt}, we have $\nabla f(\v^*, \pi^*) = -G \lambda$. Combining the two, we get $G (\lambda + \lambda') = 0$. Since $G$ is full-rank from assumption \eqref{asm:grad-tech:active-set-gradients}, we have \[G^T G (\lambda + \lambda') = 0.\]  This can be re-written as 
\[
\left [ \begin{array}{cc}
G_I^T G_I & G_I^T G_K \\
G_K^T G_I & G_K^T G_K
\end{array} \right ] \left [
\begin{array}{c}
\lambda_I + \lambda_I' \\
\lambda_K + \lambda_K'
\end{array} \right ] = 0.
\]
In other words, we have a set of simultaneous equations as follows:
\begin{eqnarray}
G_I^T G_I (\lambda_I + \lambda_I') + G_I^T G_K (\lambda_K + \lambda_K') & = & 0, \label{eq:grad-tech:lambda-condition-1}\\
G_K^T G_I (\lambda_I + \lambda_I') + G_K^T G_K (\lambda_K + \lambda_K') & = & 0. \label{eq:grad-tech:lambda-condition-2}
\end{eqnarray}

Note that at a KKT point, it is easy to see that $\lambda_K = 0$. Also, we have $\lambda_K' = 0$. So, from \eqref{eq:grad-tech:lambda-condition-1}, we have,
\[ G_I^T G_I (\lambda_I + \lambda_I') = 0.\]
Since $G_I$ is of full rank, $G_I^T G_I$ is invertible. Hence, the result.
\end{proof}

\begin{corollary}
\label{corollary:grad-tech:lambda-equal-lambda'}
Under assumption (\ref{asm:grad-tech:active-set-gradients}) of Section~\ref{sect:grad-tech:herskovits}, if $\lambda_K' = 0$ at a KKT point $\left <\v^*, \pi^* \right >$, then $\lambda = -\lambda'$.
\end{corollary}
\begin{proof}
At a KKT point, $\lambda_K = 0$. Thus, the result follows.
\end{proof}

\begin{theorem}
\label{theorem:grad-tech:lambda-zero-nash}
A KKT point $\left < \v^*, \pi^* \right >$ corresponds to a Nash equilibrium of the underlying general-sum stochastic game, if and only if $\lambda_K' = 0$.
\end{theorem}
\begin{proof} We provide the proof in two parts below.\\
\textit{If part:} From corollary \ref{corollary:grad-tech:lambda-equal-lambda'}, we have $\lambda = -\lambda'$. Let us consider a pair of constraints, $h^i(x, a^i) \le 0$, and $\pi^i(x, a^i) \ge 0$ for some $x \in \S, a^i \in \A^i(x), i = 1,2$. We consider the following cases in each of which we show that $h^i(x, a^i) \pi^i(x, a^i) = 0$ independent of the choice of $x \in \S, a^i \in \A^i(x), i = 1,2$. Thus from \eqref{eq:grad-tech:objective-product-sum}, it would follow that $f(\v^*, \pi^*) = 0$. The result then follows from theorem \ref{theorem:opt:fv2}.
\begin{enumerate}
 \item When $h^i(x, a^i) < 0$ and $\pi^i(x, a^i) = 0$ or $h^i(x, a^i) = 0$ and $\pi^i(x, a^i) > 0$ or $h^i(x, a^i) = 0$ and $\pi^i(x, a^i) = 0$ the result follows.
 \item The case when $h^i(x, a^i) < 0$ and $\pi^i(x, a^i) < 0$ does not occur. We prove this by contradiction. Suppose this case holds. Since $h^i(x, a^i) < 0$, is an inactive constraint, by complementary slackness KKT condition \eqref{subeq:grad-tech:kkt-slack}, we have that the corresponding $\lambda_j = 0 = -\lambda_j' = \pi^i(x, a^i)$. Thus, this case does not occur.
\end{enumerate}
\textit{Only if part:} If a KKT point $\left < \v^*, \pi^* \right >$ corresponds to a Nash equilibrium, then by theorem \ref{theorem:opt:fv2}, we have that $f(\v^*, \pi^*) = 0$. From equation \eqref{eq:grad-tech:objective-product-sum}, we have \[\sum_{i = 1}^2 \sum_{x \in \S} -\pi^i(x, a^i) h_i(x, a^i) = 0.\] Since a KKT point is always a feasible point of the optimization problem \eqref{eq:opt:2-player-opt}, every summand in this equation is non-negative. So, we have 
\begin{equation}
\label{eq:grad-tech:product-zero}
\pi^i(x, a^i) h_i(x, a^i) = 0, \forall x \in \S, a^i \in \A^i(x), i = 1, 2.
\end{equation}
Now from \eqref{eq:grad-tech:product-zero} and the complementary slackness KKT condition \eqref{subeq:grad-tech:kkt-slack}, it is easy to see that $\lambda_K' = 0$.
\end{proof}

\begin{definition}[KKT-N point]
A KKT point $\left < \v^*, \pi^* \right >$ of the optimization problem \eqref{eq:opt:2-player-opt} is defined to be a {\em KKT-N} (KKT-Nash) point, if the matrix \linebreak $G' = G_K^T \left ( I - G_I (G_I^T G_I)^{-1} G_I^T \right ) G_K$, computed at the KKT point, is full rank.
\end{definition}

\begin{lemma}
\label{lemma:grad-tech:sufficient-condition}
Under assumption (\ref{asm:grad-tech:active-set-gradients}) of Section~\ref{sect:grad-tech:herskovits}, a KKT-N point $\left < \v^*, \pi^* \right >$ of the optimization problem \eqref{eq:opt:2-player-opt}, corresponds to a Nash equilibrium of the underlying general-sum discounted stochastic game.
\end{lemma}
\begin{proof}
From assumption \eqref{asm:grad-tech:active-set-gradients}, we have that $G_I$ is full-rank and hence $G_I^T G_I$ is invertible. So, \eqref{eq:grad-tech:lambda-condition-1} can be simplified to get
\[ (\lambda_I + \lambda_I') = -\left (G_I^T G_I \right )^{-1} G_I^T G_K (\lambda_K + \lambda_K'). \]
Note that $\lambda_K = 0$ by definition at a KKT point. The above can be substituted in \eqref{eq:grad-tech:lambda-condition-2} and simplified to get
\[G_K^T \left ( I - G_I (G_I^T G_I)^{-1} G_I^T \right ) G_K \lambda_K' = 0.\]
Since at a KKT-N point, the matrix $G_K^T \left ( I - G_I (G_I^T G_I)^{-1} G_I^T \right ) G_K$ is full-rank, we have $\lambda_K' = 0$. Now we have the desired result from theorem \ref{theorem:grad-tech:lambda-zero-nash}.
\end{proof}

Thus, we have a sufficient condition for a KKT point to correspond to a Nash equilibrium. Note that the matrix $G'$ that needs to be of full rank is dependent on \begin{inparaenum}[(i)] \item the reward function and state transition probabilities of the underlying stochastic game, \item the value function and strategy-pair at the current KKT point, and \item the set of active and inactive constraints. \end{inparaenum} These dependencies are highly non-linear and difficult to separate. Using this sufficient condition, we can obtain a weak result on the convergence of gradient-based algorithms to Nash equilibrium solutions, as follows. Here by gradient-based algorithms, we mean those algorithms which assure convergence to a KKT point of a given optimization problem. For instance, the algorithm given in Section \ref{sect:grad-tech:scheme} is one such algorithm.

\begin{theorem}
Under assumption (\ref{asm:grad-tech:active-set-gradients}) of Section~\ref{sect:grad-tech:herskovits}, if every KKT point is also a KKT-N point, then any gradient-based algorithm when applied to the optimization problem \eqref{eq:opt:2-player-opt} would converge to a point corresponding to a Nash equilibrium of the underlying general-sum discounted stochastic game.
\end{theorem}

On the contrary, if a general-sum discounted stochastic game is such that there is at least one KKT point which is not a KKT-N point, then convergence of plain gradient-based algorithms to Nash equilibrium is not assured. %So, in general, we would need algorithms which do not just do gradient descent, but rather use ``some more'' properties of the problem. %In the quest for obtaining such properties, we construct different necessary and sufficient conditions for Nash equilibria in the next chapter. %To illustrate the fact that the optimization problem \eqref{eq:opt:2-player-opt} can 

\section{Conclusion}
\label{sect:conclusion}

We first proposed a simple gradient descent scheme for solution of general-sum stochastic games. During the construction of the scheme, we discussed the overall nature of the indefinite objective and non-convex constraints illustrating the fact that a simple steepest descent algorithm may not even converge to a local minimum of the optimization problem. The proposed scheme takes these issues while constructing both \begin{inparaenum}[(i)] \item feasible search direction as well as, \item optimal step-length. \end{inparaenum} Also, it tries to address numerical efficiency by appropriately using sparsity techniques for an associated matrix inversion. We observed that the size of the optimization problem increases exponentially in the number of variables and the number of constraints. We showed that the proposed scheme converges to a KKT point of the optimization problem. This was seen to be sufficient in simulations performed for the example problem of terrain exploration. However, in general, we showed in Section~\ref{sect:grad-tech:fallacies} that it may not be sufficient for a scheme to converge to any KKT point as the same may not correspond to a Nash equilibrium. The results discussed in Section \ref{sect:grad-tech:fallacies} can be easily generalized to the case where there are more than two players. In summary, usual gradient schemes could possibly suffer from two issues: \begin{inparaenum}[(i)] \item Non-convergence to Nash equilibria which is the more serious of the two issues, and \item scalability to higher problem sizes. \end{inparaenum} In would be interesting to derive gradient-based algorithms that provide guaranteed convergence to Nash equilibria.

\bibliographystyle{plainnat}
\bibliography{references}

\end{document}